\documentclass{article}

\PassOptionsToPackage{table}{xcolor}
\usepackage{iclr2027_conference,times}

\usepackage[utf8]{inputenc}
\usepackage[T1]{fontenc}
\usepackage{amsmath,amssymb,amsfonts,amsthm,mathtools,bm}
\usepackage{booktabs,multirow,array}
\usepackage{graphicx}
\usepackage{algorithm}
\usepackage{algpseudocode}
\usepackage{microtype}
\usepackage{nicefrac}
\usepackage{url}
\usepackage{float}
\usepackage{placeins}
\usepackage{hyperref}

\usepackage{wrapfig}
\usepackage{adjustbox}

\usepackage{booktabs}
\usepackage{tabularx}
\usepackage{makecell}
\usepackage{threeparttable}
\usepackage{amssymb}

\usepackage[nameinlink,capitalise]{cleveref}

\DeclareMathOperator*{\argmin}{arg\,min}
\DeclareMathOperator{\median}{median}
\DeclareMathOperator{\clip}{clip}

\newcommand{\ind}{\mathbf{1}}
\newcommand{\E}{\mathbb{E}}
\newcommand{\Pp}{\mathbb{P}}
\newcommand{\R}{\mathbb{R}}
\newcommand{\Normal}{\mathcal{N}}
\newcommand{\RoVR}{\operatorname{RoVR}}
\newcommand{\SoftRoVR}{\operatorname{SoftRoVR}}

\newtheoremstyle{rovrplain}
{5pt}{5pt}{\itshape}{}{\normalfont}{.}{0.5em}{}
\newtheoremstyle{rovrdefinition}
{5pt}{5pt}{\normalfont}{}{\normalfont}{.}{0.5em}{}
\theoremstyle{rovrplain}
\newtheorem{theorem}{Theorem}[section]

\newtheorem{proposition}[theorem]{Proposition}
\newtheorem{corollary}[theorem]{Corollary}
\theoremstyle{rovrdefinition}

\newtheorem{assumption}[theorem]{Assumption}

\title{\raggedright Dual-Channel Robust Group-Relative Policy Optimization via Advantage and Sequence-Weight Estimation}

\author{
Zhongyi Li\textsuperscript{1*}\ \ 
Wan Tian\textsuperscript{2*} \ \ 
Xiang Xu\textsuperscript{1} \ \ 
Yutian Xiao\textsuperscript{1}\ \ 
Yikun Ban\textsuperscript{1} \ \ 
Yijie Peng\textsuperscript{3$\dagger$}\ \ 
Fuzhen Zhuang\textsuperscript{1$\dagger$} \\[1ex]
\textsuperscript{1}Beihang University \quad
\textsuperscript{2}Peking University \quad
\textsuperscript{3}Nanjing University \\[0.5ex]
\textsuperscript{*}These authors contributed equally to this work.\quad
\textsuperscript{$\dagger$}Corresponding authors. \\
\texttt{Correspondence: pengyijie@nju.edu.cn; zhuangfuzhen@buaa.edu.cn}
}

\iclrfinalcopy

\begin{document}

\maketitle

\begin{abstract}
Group-relative policy optimization relies on reward-derived advantages and sequence-level likelihood weights, both of which can be sensitive to localized outliers. Extreme rewards can collapse the contrast among clean responses after group normalization, while token-level log-ratio perturbations can alter sequence weights and clipping decisions. We introduce RoVR-GSPO, a dual-channel robust optimizer that addresses these failure modes separately. Its reward channel combines robust reference estimation with bounded residual credit, while its ratio channel uses differentiable SoftRoVR aggregation to construct robust sequence weights. We provide stability and efficiency analyses for both channels. Experiments on mathematical reasoning, long-context summarization, and tool-call annotation show consistent improvements over GSPO, while controlled perturbation studies demonstrate stronger robustness to reward contamination and token-ratio anomalies.
\end{abstract}

\section{Introduction}
\label{sec:introduction}

Reinforcement learning is central to language-model post-training~\citep{ouyang2022instructgpt,lightman2023verify,xu2025towards}. PPO uses a learned value function~\citep{schulman2017ppo}, whereas GRPO removes the critic by comparing several responses sampled for the same prompt~\citep{shao2024deepseekmath}. DAPO and GSPO retain this group-relative structure while changing training and likelihood-correction choices~\citep{yu2025dapo,zheng2025gspo}. These methods reduce the need for a learned critic, but they still depend on small-sample statistics: one reward reference determines all response advantages, and a sequence-level likelihood weight summarizes many token log-ratios.

The reward-side failure is more specific than the statement that rewards are heavy-tailed. If one reward is replaced by $R_j+\Delta$, the mean moves by $\Delta/G$ while the RMS grows with $|\Delta|$. The standardized advantages stay bounded, but as $|\Delta|\to\infty$ the anomalous response approaches $\operatorname{sign}(\Delta)\sqrt{G-1}$ and all clean responses approach $-\operatorname{sign}(\Delta)/\sqrt{G-1}$. Their pairwise advantage differences therefore vanish. This separates two roles that are often conflated: a robust reference limits shared displacement, whereas a bounded residual map protects focal leverage and the scale used to normalize clean responses.

We use this separation as the design principle for RoVR-GSPO. On the reward channel, RoVR computes a robust reference and then applies a bounded odd map and mapped RMS, yielding $\widehat A^{\mathrm{credit}}$. On the ratio channel, SoftRoVR applies differentiable robust aggregation to token log-ratios and supplies sequence weights to the clipped GSPO objective. The method therefore exposes the reference, credit, and sequence-weight operations as distinct objects that can be analyzed and evaluated separately. Our contributions are:

\begin{figure}[!htbp]
\centering
\includegraphics[width=0.98\linewidth]{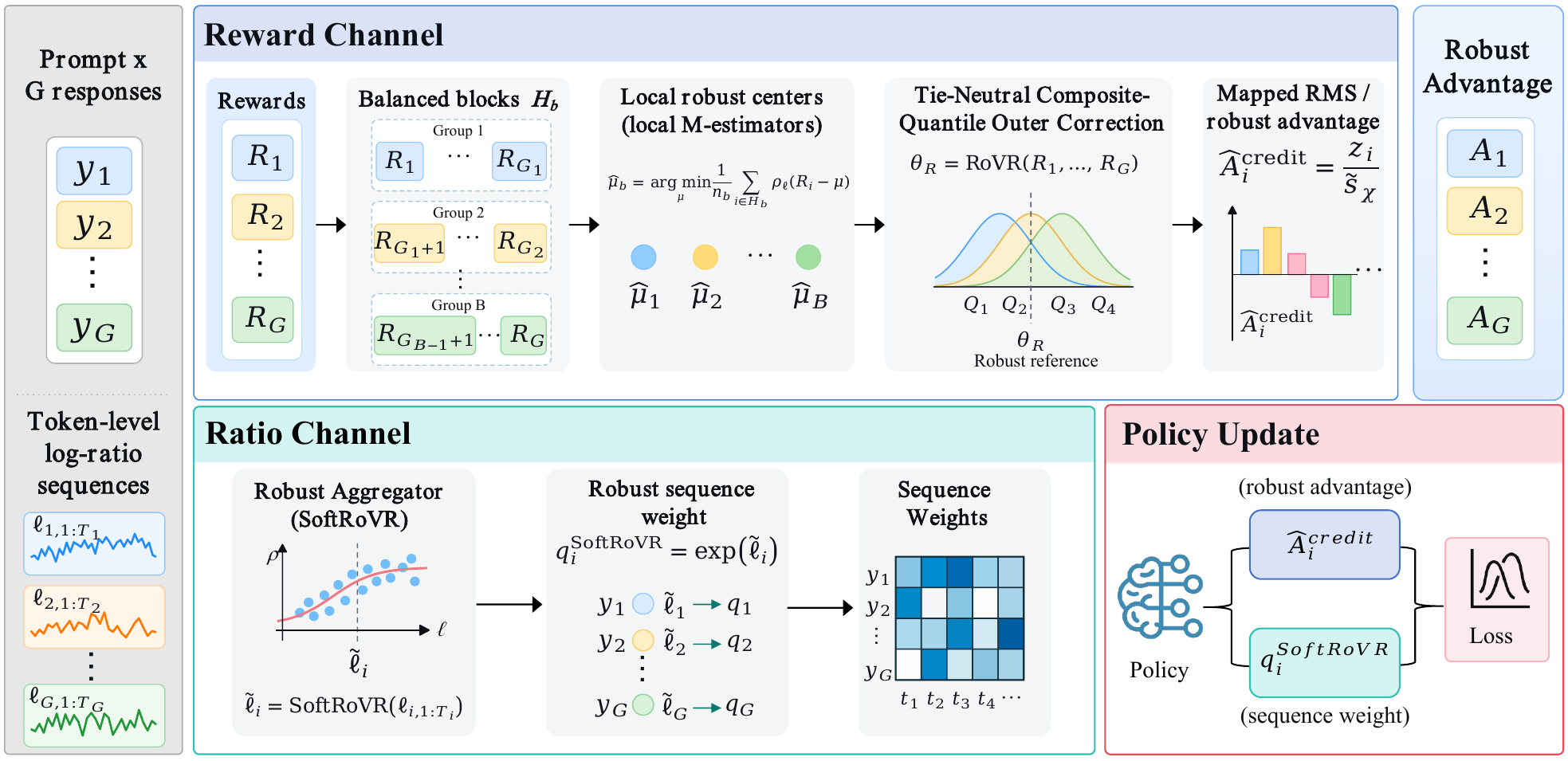}
\caption{Overview of RoVR-GSPO. The reward channel first estimates a robust reference, then applies the bounded residual map and mapped RMS to form $\widehat A^{\mathrm{credit}}$. The ratio channel converts token-level log-ratio sequences into differentiable robust sequence weights via SoftRoVR. These two signals jointly determine the policy update.}
\label{fig:frame}
\end{figure}

\begin{itemize}
\item We identify two distinct update sensitivities: reward perturbations can erase clean-response advantage contrast, while localized token log-ratio perturbations can change sequence weights and clipping branches.
\item We introduce RoVR-GSPO, which combines robust reference estimation with bounded credit on the reward channel and a differentiable robust sequence surrogate on the ratio channel. This separation makes the role of each operation explicit.
\item We derive conditional reference, credit-stability, outer-efficiency, and clipping-stability results, and connect them to downstream comparisons, channel placement, and paired perturbation studies.
\end{itemize}

Here ``variance reduction'' refers to the outer reference-estimation factor relative to median aggregation. SoftRoVR is a sequence-weight surrogate whose robustness properties are studied directly; it is not treated as an exact importance ratio.

\section{Background: From Group Statistics to Policy Updates}
\label{sec:background}

RoVR-GSPO is motivated by two distinct sensitivities: reward perturbations can collapse clean-response advantage contrast, while token log-ratio bursts can alter sequence weights and clipping branches. We analyze these effects below.

\subsection{Mean--RMS group-relative advantages}

For a prompt $x$, let the old policy sample responses $\{y_i\}_{i=1}^{G}$ with rewards $\{R_i\}_{i=1}^{G}$. GRPO constructs
\begin{equation}
\bar R = \frac{1}{G}\sum_{j=1}^{G}R_j,\qquad
s_R^2 = \frac{1}{G}\sum_{j=1}^{G}(R_j-\bar R)^2+\varepsilon_s,
\qquad
\widehat A_i^{\mathrm{GRPO}}=\frac{R_i-\bar R}{s_R},
\label{eq:grpo_advantage}
\end{equation}
where $\varepsilon_s\geq0$. The next proposition shows that self-normalization masks anomaly magnitude but broadcasts its effect across the group.

\begin{proposition}
\label{prop:grpo_broadcast}
Let $G\geq2$. Fix $R_1,\ldots,R_G$, replace $R_j$ by $R_j+\Delta$, and consider values of $\Delta$ for which $s_\Delta>0$ (in particular, all $\Delta$ when $\varepsilon_s>0$). Let $\bar R_\Delta,s_\Delta,A_{i,\Delta}$ denote the resulting quantities. Then
\begin{equation}
\bar R_\Delta=\bar R+\frac{\Delta}{G},\qquad
s_\Delta^2=s_R^2+\frac{2\Delta(R_j-\bar R)}{G}
+\frac{G-1}{G^2}\Delta^2.
\label{eq:grpo_spike_exact}
\end{equation}
For every finite group, $\sum_iA_{i,\Delta}=0$, $\sum_iA_{i,\Delta}^2\leq G$, and $\max_i|A_{i,\Delta}|\leq\sqrt{G-1}$. Moreover,
\begin{equation}
A_{j,\Delta}\longrightarrow\operatorname{sign}(\Delta)\sqrt{G-1},
\qquad
A_{i,\Delta}\longrightarrow-\frac{\operatorname{sign}(\Delta)}{\sqrt{G-1}}\quad(i\neq j)
\label{eq:grpo_broadcast_limit}
\end{equation}
as $|\Delta|\to\infty$. Consequently, when $G\geq3$, any two clean responses $i,k\neq j$ satisfy $A_{i,\Delta}-A_{k,\Delta}=(R_i-R_k)/s_\Delta\to0$.
\end{proposition}

For a clean index set $H$ with $|H|\geq2$, define the retained pairwise advantage contrast
\(
\mathcal C_H(A)=
\{
2\sum_{i<k,\,i,k\in H}(A_i-A_k)^2 / |H|(|H|-1)
\}^{1/2}.
\)
Equation \eqref{eq:grpo_broadcast_limit} gives $\mathcal C_H(A_\Delta)\to0$ when $H$ contains the fixed clean responses. This quantity is measurable in saved batches and is therefore a primary diagnostic rather than an informal interpretation.

\begin{proposition}
\label{prop:center_only_impossibility}
Let a location estimate $m_\Delta$ remain bounded when one reward is $R_j+\Delta$ and all other rewards are fixed. Form raw-residual credits $C_{i,\Delta}=(R_i^{(\Delta)}-m_\Delta)/d_\Delta$ with $d_\Delta>0$. If $d_\Delta=O(1)$, then $|C_{j,\Delta}|\to\infty$. If $C_{j,\Delta}=O(1)$, then $d_\Delta=\Omega(|\Delta|)$ and $C_{i,\Delta}\to0$ for every clean $i\neq j$. In particular, an RMS about $m_\Delta$ gives $d_\Delta\sim|\Delta|/\sqrt G$, $C_{j,\Delta}\to\operatorname{sign}(\Delta)\sqrt G$, and clean-advantage erasure.
\end{proposition}

For raw-residual credits with a shared scale, protecting both focal leverage and clean-response contrast therefore requires an additional residual transformation. We use the bounded map in \Cref{sec:reward_credit}, which also makes its clean-signal fidelity explicit.

\subsection{Sequence weights and clipping asymmetry}

For response length $T_i$, define token log ratios
\(
\ell_{i,t}(\theta)
=\log\pi_\theta(y_{i,t}\mid x,y_{i,<t})
-\log\pi_{\mathrm{old}}(y_{i,t}\mid x,y_{i,<t}).
\)
GSPO uses
\begin{equation}
q_i^{\mathrm{GSPO}}(\theta)
=\exp\!\left(\frac{1}{T_i}\sum_{t=1}^{T_i}\ell_{i,t}(\theta)\right).
\label{eq:gspo_ratio}
\end{equation}
With $l=1-\epsilon_{\mathrm{lo}}>0$ and $u=1+\epsilon_{\mathrm{hi}}$, the per-sequence clipped term is
\begin{equation}
g(q,A)=\min\{qA,\clip(q,l,u)A\},
\qquad
\mathcal L_{\mathrm{GSPO}}(\theta)
=-\E\!\left[\frac1G\sum_{i=1}^{G}g(q_i(\theta),\widehat A_i)\right].
\label{eq:gspo_objective}
\end{equation}
Writing $q=e^m$, away from $q\in\{l,u\}$,
\begin{equation}
\frac{\partial g(e^m,A)}{\partial m}
=
\begin{cases}
e^mA, & A>0,\ e^m<u,\\
0,    & A>0,\ e^m>u,\\
0,    & A<0,\ e^m<l,\\
e^mA, & A<0,\ e^m>l.
\end{cases}
\label{eq:clipping_branch_derivative}
\end{equation}
Therefore an upper-ratio excursion is clipped for positive advantage but remains on the $qA$ branch for negative credit. If a contiguous token set $\mathcal C$ receives perturbations $\delta_t$, then
\begin{equation}
\frac{q'}{q}=\exp\!\left(\frac1T\sum_{t\in\mathcal C}\delta_t\right),
\label{eq:burst_ratio}
\end{equation}
If $q<u$ initially, the perturbation crosses the upper boundary whenever $T^{-1}\sum_{t\in\mathcal C}\delta_t>\log(u/q)$. If $q\geq u$ already, the same upward perturbation magnifies the active negative-advantage branch. Thus localized token perturbations can affect both the sequence weight and the active optimization branch.

\section{Dual-Channel Robust Estimation for Group-Relative Policy Optimization}
\label{sec:method}

RoVR-GSPO robustifies the two compact statistics that drive a GSPO update. For prompt-level rewards $R_{1:G}$, the reward channel separates reference estimation from credit shaping: it computes a hard RoVR reference and converts centered residuals into the bounded normalized advantage $\widehat A^{\mathrm{credit}}$. For the token log-ratios $\ell_{i,1:T_i}$ of each response, the ratio channel applies SoftRoVR to obtain a differentiable sequence-weight surrogate. The complete update is therefore
\(
(R_{1:G},\,\ell_{i,1:T_i})
\longrightarrow
(\widehat A_i^{\mathrm{credit}},\,q_i^{\mathrm{SoftRoVR}})
\longrightarrow
\mathcal L_{\mathrm{GSPO}}.
\)
Building on the composite-quantile aggregation principle used in VRMOM~\citep{tu2021variancereducedmedianofmeansestimator}, RoVR uses bounded-score local centers and an explicit credit map so that reference protection and clean-contrast protection remain distinguishable. We reserve \emph{RoVR} for the hard reward-reference estimator, \emph{SoftRoVR} for the differentiable ratio-side aggregator, and \emph{RoVR-GSPO} for the complete dual-channel optimizer. The single-channel variants serve as placement ablations in \Cref{sec:robust_scale_exp}.

\subsection{RoVR robust reference estimation}

The main method fixes Barron's shape to $\alpha=1$~\citep{BarronCVPR2019}. Its loss and score are
\begin{equation}
\rho_c(u)=\sqrt{1+(u/c)^2}-1,
\qquad
\psi_c(u)=\frac{u}{c^2}\left(1+\frac{u^2}{c^2}\right)^{-1/2},
\qquad
\psi_c'(u)=\frac{1}{c^2}\left(1+\frac{u^2}{c^2}\right)^{-3/2}.
\label{eq:convex_score}
\end{equation}
The empirical objective is strictly convex, its score equation has a unique root, and $M_c:=\sup_u|\psi_c(u)|=1/c$. We fix the convex shape $\alpha=1$ throughout the main experiments and partition $N$ scalar observations into $B\geq1$ balanced blocks $\{\mathcal H_b\}_{b=1}^{B}$ of sizes $n_b$, where $\sum_bn_b=N$ and $|n_b-n_{b'}|\leq1$. The local robust center is
\begin{equation}
\widetilde\mu_b
=\argmin_{u\in\R}
\frac{1}{n_b}\sum_{i\in\mathcal H_b}\rho_c(x_i-u).
\label{eq:block_m}
\end{equation}
Its population target is the M-functional
\begin{equation}
\theta_\rho(P):\qquad \E_P[\psi_c(X-\theta_\rho)]=0.
\label{eq:m_functional}
\end{equation}
For a symmetric distribution, $\theta_\rho$ equals its center and, when finite, its mean. Under skewness it is an explicit robust reference rather than an undisclosed mean estimator. For each block, define its empirical sandwich scale by
\[
\widehat a_b=\frac1{n_b}\sum_{i\in\mathcal H_b}\psi_c'(x_i-\widetilde\mu_b),
\qquad
\widehat b_b=\frac1{n_b}\sum_{i\in\mathcal H_b}\psi_c^2(x_i-\widetilde\mu_b),
\qquad
\widehat\nu_b=\frac{\sqrt{\widehat b_b}}{\widehat a_b\vee a_{\min}}.
\]
We pool and cap these local sandwich scales,
\begin{equation}
\widehat\nu=
\clip\!\left(\median_b\widehat\nu_b,\nu_{\min},\nu_{\max}\right).
\label{eq:robust_sandwich_scale}
\end{equation}
For blocks of only two or three observations, a blockwise sandwich estimate can be unstable; the implementation must then use an independent pilot or a lagged running scale and report that choice. The caps are numerical safeguards, not quantities known from theory.

The outer correction is a tie-neutral composite-quantile update. For even $B$, the initializer is the midpoint of the two central local centers; for odd $B$, it is the central local center. Let $\widehat\mu_0=\median_b\widetilde\mu_b$ and define
\[
\tau_k=\frac{k}{K+1},\qquad
\Delta_k=\Phi^{-1}(\tau_k),\qquad
D_K=\sum_{k=1}^{K}\phi(\Delta_k),\qquad
W_B=\sum_{b=1}^{B}\sqrt{n_b}.
\]
Discrete rewards require a tie convention. We use the mid-indicator
\begin{equation}
J_0(u)=\ind\{u<0\}+\frac12\ind\{u=0\}.
\label{eq:mid_indicator}
\end{equation}
The RoVR robust reference is
\begin{equation}
\widehat\theta_{\mathrm{RoVR}}
=\widehat\mu_0-
\frac{\widehat\nu}{D_KW_B}
\sum_{b=1}^{B}\sum_{k=1}^{K}
\left[
J_0\!\left(
\widetilde\mu_b-\widehat\mu_0-
\frac{\widehat\nu\Delta_k}{\sqrt{n_b}}
\right)-\tau_k
\right].
\label{eq:rovr}
\end{equation}
For equal $n_b=n$, \eqref{eq:rovr} reduces to the usual prefactor $\widehat\nu/(B\sqrt nD_K)$. The denominator $W_B$ is the Newton derivative implied by heteroscedastic normal block centers. Its associated effective sample size is
\begin{equation}
N_{\mathrm{eff}}=\frac{W_B^2}{B}\leq N,
\label{eq:effective_sample_size}
\end{equation}
which equals $N$ for equal blocks. For balanced blocks, $N_{\mathrm{eff}}/N\to1$ when $n_{\min}\to\infty$; \eqref{eq:effective_sample_size} records the exact finite-block loss otherwise.

\begin{proposition}
\label{prop:calibration}
For every $a\in\R$, the hard RoVR map is translation equivariant and constant preserving:
\[
\RoVR(x_1+a,\ldots,x_N+a)=\RoVR(x_1,\ldots,x_N)+a,
\qquad
\RoVR(a,\ldots,a)=a.
\]
When $B=1$, it equals the single convex M-center in \eqref{eq:block_m}. These statements hold for even or odd $K$ and for discrete inputs.
\end{proposition}

The half weight at a tie is essential. With the conventional ``$\leq$'' indicator, odd $K$ gives a spurious correction of $-\widehat\nu B/(2D_KW_B)$ on constant data. The logistic relaxation below takes value $1/2$ at zero, so \eqref{eq:mid_indicator} also aligns the hard and smooth estimators. After the local roots are computed, direct evaluation costs $O(BK)$; the equal-block version admits the same rank-count simplification as VRMOM~\citep{tu2021variancereducedmedianofmeansestimator}.

\long\def\rovrReferenceEstimatorTable{
	\begin{table}[H]
		\centering
		\caption{
			Comparison of representative location/reference estimators.
			RoVR combines a bounded local $M$-estimator with tie-neutral
			composite-quantile aggregation while retaining explicit block protection.
		}
		\label{tab:reference_estimators}
		
		\small
		\setlength{\tabcolsep}{4pt}
		\renewcommand{\arraystretch}{0.90}
		
		\begin{threeparttable}
			\begin{tabularx}{\textwidth}{
					@{}
					l
					>{\raggedright\arraybackslash}X
					>{\raggedright\arraybackslash}X
					ccc
					@{}
				}
				\toprule
				\textbf{Estimator}
				&
				\textbf{Local statistic}
				&
				\textbf{Outer aggregation}
				&
				\makecell{\textbf{Bounded}\\[-1pt]\textbf{local score}}
				&
				\makecell{\textbf{Block}\\[-1pt]\textbf{protection}}
				&
				\makecell{\textbf{Clean outer}\\[-1pt]\textbf{factor}}
				\\
				\midrule
				
				Arithmetic mean
				& Mean
				& None
				& --
				& --
				& $1$
				\\
				
				Global $M$-center
				& Bounded-score $M$-center
				& None
				& $\checkmark$
				& --
				& --
				\\
				
				MOM
				& Arithmetic block mean
				& Median
				& --
				& $\checkmark$
				& $\pi/2$
				\\
				
				VRMOM
				& Arithmetic block mean
				& Composite quantile
				& --
				& $\checkmark$
				& $V_K \to \pi/3$
				\\
				
				Robust MOM
				& Robust local $M$-center
				& Median
				& $\checkmark$
				& $\checkmark$
				& $\pi/2$
				\\
				
				\textbf{RoVR}
				& \textbf{Robust local $M$-center}
				& \textbf{Tie-neutral composite quantile}
				& $\boldsymbol{\checkmark}$
				& $\boldsymbol{\checkmark}$
				& $\boldsymbol{V_K \to \pi/3}$
				\\
				
				\bottomrule
			\end{tabularx}
		\end{threeparttable}
	\end{table}
}

\Cref{tab:reference_estimators} summarizes the reference estimators. Relative to MOM or Robust MOM, RoVR replaces the outer median by a symmetric composite-quantile correction. Under the clean local normal-score model, this changes the outer variance factor from $\pi/2$ to $V_K$, with $V_K\to\pi/3$ as the quantile grid is refined. The complete variance combines this factor with the local sandwich variance. RoVR shares the outer-efficiency principle of VRMOM and adds bounded-score local centers, tie-neutral handling of discrete inputs, and balanced unequal blocks.

\subsection{Reward-channel advantage estimation}
\label{sec:reward_credit}

The robust reference controls the shared reward baseline, while the credit map controls the focal residual and the scale used for normalization. RoVR therefore uses the following bounded odd residual map when forming reward advantages:
\begin{equation}
\chi_\kappa(u)=\frac{u}{\sqrt{1+(u/\kappa)^2}}
=\kappa^2\psi_\kappa(u),
\qquad
|\chi_\kappa(u)|\leq\kappa,\qquad 0<\chi_\kappa'(u)\leq1.
\label{eq:bounded_credit_map}
\end{equation}
It preserves sign and local slope while saturating smoothly. With $\widehat\theta_R=\RoVR(R_1,\ldots,R_G)$, define
\begin{equation}
\begin{aligned}
\widehat s_\chi
&=\{s_{\min}^2+\frac1G\sum_{j=1}^{G}
\chi_\kappa^2(R_j-\widehat\theta_R)\}^{1/2},\quad 
\widehat A_i^{\mathrm{credit}}
=\frac{\chi_\kappa(R_i-\widehat\theta_R)}{\widehat s_\chi}.
\end{aligned}
\label{eq:credit_advantage}
\end{equation}
The main bounded advantage is deterministically bounded,
\begin{equation}
|\widehat A_i^{\mathrm{credit}}|
\leq\min\!\left\{\frac{\kappa}{s_{\min}},\sqrt G\right\}.
\label{eq:credit_global_bound}
\end{equation}
The map is close to the raw residual on clean moderate rewards:
\begin{equation}
|\chi_\kappa(u)-u|
\leq\frac{|u|^3}{2\kappa^2}.
\label{eq:credit_clean_fidelity}
\end{equation}
Thus $\kappa$ controls the fidelity--robustness trade-off. For bounded rule rewards, choosing $\kappa$ above the reward range makes the distortion small. All reward-side statistics are detached before they enter the policy loss. The center-only comparison below retains the raw residual and raw RMS while replacing only the arithmetic mean by $\widehat\theta_R$:
\(
\widehat A_i^{\mathrm{center}}
=\frac{R_i-\widehat\theta_R}
{\{\varepsilon_s+G^{-1}\sum_j(R_j-\widehat\theta_R)^2\}^{1/2}}.
\)
This isolates the effect of the RoVR reference from the bounded residual map and mapped scale. For an action-independent control-variate reference, use a linear leave-one-out numerator and a positive leave-one-out scale $S_{-i}$, both computed without $R_i$:
\begin{equation}
\widehat A_i^{\mathrm{LOO}}
=\frac{R_i-\widehat\theta_{R,-i}}{S_{-i}},
\qquad
\widehat\theta_{R,-i}=\RoVR(R_1,\ldots,R_{i-1},R_{i+1},\ldots,R_G).
\label{eq:loo_advantage}
\end{equation}
A bounded residual leave-one-out version is a shaped robust surrogate and does not inherit the linear baseline cancellation. Its deviation is quantified in \Cref{prop:credit_gradient_distortion}.


\long\def\rovrAdvantageVariantTable{
\begin{table}[!htbp]
\centering
\caption{
RoVR-based advantage constructions used in our analysis and algorithm.
}
\label{tab:advantage_variants}

\small
\setlength{\tabcolsep}{4.5pt}
\renewcommand{\arraystretch}{1.15}

\begin{threeparttable}
\begin{tabularx}{\linewidth}{
@{}
l
l
l
l
>{\raggedright\arraybackslash}X
@{}
}
\toprule
\textbf{Variant}
&
\textbf{RoVR reference}
&
\textbf{Numerator}
&
\textbf{Scale}
&
\textbf{Role}
\\
\midrule

$\widehat A_i^{\mathrm{center}}$
&
$\operatorname{RoVR}(R_{1:G})$
&
$R_i-\widehat\theta_R$
&
Raw RMS
&
Center-only ablation
\\

$\widehat A_i^{\mathrm{LOO}}$
&
$\operatorname{RoVR}(R_{-i})$
&
$R_i-\widehat\theta_{R,-i}$
&
$S_{-i}$
&
Leave-one-out theoretical reference
\\

\addlinespace[2pt]
\textbf{$\widehat A_i^{\mathrm{credit}}$}
&
$\operatorname{RoVR}(R_{1:G})$
&
$\boldsymbol{\chi_\kappa(R_i-\widehat\theta_R)}$
&
$\widehat s_{\chi}$
&
\textbf{Default robust advantage}; bounded leverage and fidelity control
\\

\bottomrule
\end{tabularx}

\begin{tablenotes}[flushleft]
\footnotesize
\item
Only $\widehat A_i^{\mathrm{credit}}$ is used by default in RoVR-GSPO.
$\widehat A_i^{\mathrm{center}}$ is a center-only ablation, whereas
$\widehat A_i^{\mathrm{LOO}}$ provides a leave-one-out reference for the
theoretical analysis.
\end{tablenotes}

\end{threeparttable}
\end{table}
}

\Cref{tab:advantage_variants} separates the roles of the three advantage constructions: $\widehat A^{\mathrm{center}}$ isolates reference substitution, $\widehat A^{\mathrm{LOO}}$ provides a linear action-independent baseline for analysis, and $\widehat A^{\mathrm{credit}}$ is the normalized bounded advantage used by RoVR-GSPO.

\subsection{SoftRoVR sequence-weight estimation for GSPO}

The reward statistic uses the hard RoVR operator in \eqref{eq:rovr}. The ratio-side differentiable surrogate is called SoftRoVR. For token log-ratios, define
\(
H_\gamma(u)=\frac{1}{1+\exp(u/\gamma)}, \gamma>0,
\)
which converges pointwise to $J_0$, including at ties. A differentiable outer initializer can be obtained from
\begin{equation}
\operatorname{smed}_{\eta}(v_1,\ldots,v_B)
=\argmin_u\frac1B\sum_{b=1}^{B}\rho_\eta(v_b-u),
\qquad \eta>0.
\label{eq:smooth_median}
\end{equation}
Let $\mathcal S_\gamma$ retain the hard block centers, midpoint-median initializer, and pooled scale, replacing only $J_0$ by $H_\gamma$. This auxiliary map isolates indicator-smoothing error. The implemented surrogate $\SoftRoVR_{\gamma,\eta,S}$ additionally uses \eqref{eq:smooth_median} for both median operations and unrolls $S$ safeguarded solver iterations. At fixed $\eta>0$, the smooth objective specifies a unique initializer, including for even $B$. We analyze this finite-temperature surrogate directly; comparison with hard RoVR separates indicator, initializer, scale, and solver errors. The zero-temperature indicator limit agrees with $J_0$ at ties.

For response $i$, the robust sequence weight is
\begin{equation}
q_i^{\mathrm{SoftRoVR}}(\theta)
=\exp\!\left[
\SoftRoVR_{\gamma,\eta,S}
\bigl(\ell_{i,1}(\theta),\ldots,\ell_{i,T_i}(\theta)\bigr)
\right].
\label{eq:softrovr_ratio}
\end{equation}
RoVR-GSPO combines $\widehat A^{\mathrm{credit}}$ from \eqref{eq:credit_advantage} with \eqref{eq:softrovr_ratio}. Disabling one channel gives a placement ablation. Robust aggregation defines a sequence-weight surrogate with a different likelihood functional; its calibration and local stability are analyzed in \Cref{sec:theory_summary}.

\subsection{Resolution, failure geometry, and fallback}

The block design determines the contamination budget. To tolerate $q$ arbitrary blocks while retaining a strict majority of usable blocks, one needs $B\geq2q+1$. To leave a strict clean majority after at most $s$ replacements inside each usable block, one needs $n_{\min}\geq2s+1$. Hence a non-degenerate worst-case two-level design requires
\begin{equation}
N\geq(2q+1)(2s+1).
\label{eq:resolution_frontier}
\end{equation}
This condition describes the resolution needed for simultaneous worst-case protection at both levels. In particular, $G<9$ permits budgets with $q=0$ or $s=0$, but cannot support both $q\geq1$ and $s\geq1$. RoVR itself is defined for every balanced partition; the budget determines the applicable guarantee. The $B=1$ case provides a global-M boundary case. Contiguous token blocks encode a burst geometry, while random partitions support the independently assigned contamination model.

\long\def\rovrUpdateAlgorithm{
\begin{algorithm}[!htbp]
\caption{RoVR-GSPO: dual-channel robust GSPO update}
\label{alg:rovr_gspo}
\begin{algorithmic}[1]
\Require rewards $R_{1:G}$; proposed $B$ and balanced sizes $\{n_b\}$; declared failure unit and budgets $(q,s)$
\Require $(K,c,\kappa,s_{\min},a_{\min},\nu_{\min},\nu_{\max})$; SoftRoVR switch $z_{\rm ratio}$
\State $z_{\rm block}\gets$ failure unit supported and budgets $(q,s)$ declared
\State $z_{\rm block}\gets z_{\rm block}\land(B>1)\land(B\geq2q+1)$
\State $z_{\rm block}\gets z_{\rm block}\land(\min_bn_b\geq2s+1)\land(\sum_bn_b=G)$
\State $z_{\rm block}\gets z_{\rm block}\land(\max_bn_b-\min_bn_b\leq1)$
\If{$z_{\rm block}=0$ or the realized assignment violates the declared sizes}
\State Set $B\gets1$ and $\mathcal H_1\gets\{1,\ldots,G\}$ \Comment{global-M fallback}
\Else
\State Use the predeclared balanced assignment sampled independently of corruption labels used by the guarantee
\EndIf
\State Compute every convex local center by \eqref{eq:block_m}
\If{$B=1$}
\State $\widehat\theta_R\gets\widetilde\mu_1$
\Else
\State Compute the capped pooled scale \eqref{eq:robust_sandwich_scale} and RoVR reference \eqref{eq:rovr}
\EndIf
\State $z_i\gets\chi_\kappa(R_i-\widehat\theta_R)$ and
$\widehat s_\chi\gets\{s_{\min}^2+G^{-1}\sum_jz_j^2\}^{1/2}$
\State $\widehat A_i^{\mathrm{credit}}\gets z_i/\widehat s_\chi$ for $i=1,\ldots,G$; detach reward statistics
\If{$z_{\rm ratio}=1$}
\State Require a GSPO parent; compute the SoftRoVR sequence weight in \eqref{eq:softrovr_ratio}
\Else
\State Keep the canonical GSPO sequence aggregate
\EndIf
\State \Return $\widehat A_{1:G}^{\mathrm{credit}}$ and the selected sequence weights
\end{algorithmic}
\end{algorithm}
}

\section{Theoretical Guarantees}
\label{sec:theory_summary}

We analyze the three quantities modified by RoVR-GSPO: the reward reference, bounded response credit, and sequence weight. The results cover finite-group stability, clean outer efficiency, calibration, and clipping margins; auxiliary statements and proofs appear in \Cref{app:theory_statements,app:finite_group,app:target_scale,app:policy}.

Consider $B$ declared blocks, with a fixed set of $q<r_B:=\lceil B/2\rceil$ arbitrary blocks and honest complement $\mathcal H$. Each $b\in\mathcal H$ contains $n_b$ independent clean draws followed by at most $s_b$ adaptive replacements. If $g(t)=\E\psi_c(X-t)$ satisfies $g(\theta+u)\leq-au$ and $g(\theta-u)\geq au$ for $0\leq u\leq r_0$, define, for $0<t\leq r_0$,
\[
\begin{gathered}
H=B-q,\quad \eta_q=(r_B-q)/(B-q),\quad
p_b(t)=\exp\!\left[-\frac{n_b}{2M_c^2}
\left(at-\frac{2s_bM_c}{n_b}\right)_+^2\right],\\
\bar p_H(t)=H^{-1}\!\sum_{b\in\mathcal H}p_b(t),\quad
\beta_H(t)=\min\{1,2e^{-2H(\eta_q-\bar p_H(t))^2}\},\\
\widehat C_B=|\widehat\theta_{\mathrm{RoVR}}-\widehat\mu_0|,
\qquad C_B=\nu_{\max}BK/(2D_KW_B).
\end{gathered}
\]
Only the population score separation is distributional; replacement values may be unbounded.

\begin{theorem}[Finite-group reference and credit stability]
\label{thm:main_finite_credit}
If $\bar p_H(t)<\eta_q$, then $\widehat C_B\leq C_B$ and
\begin{equation}
\Pp\!\left(|\widehat\theta_{\mathrm{RoVR}}-\theta|>t+\widehat C_B\right)
\leq\beta_H(t),\qquad
\Pp\!\left(|\widehat\theta_{\mathrm{RoVR}}-\theta|>t+C_B\right)
\leq\beta_H(t).
\label{eq:main_finite_group}
\end{equation}
Suppose additionally that the observed reward group and its clean precursor differ in at most $r$ coordinates. Let $A_i^{\theta,\star}$ be the bounded credit of the precursor computed with population reference $\theta$, and set $e_\theta(t)=t+C_B$, $e_{R,i}=|R_i-R_i^\star|$, and
\(
\bar e_s(t)=\{r\kappa^2/G+2\kappa e_\theta(t)\}/(2s_{\min}).
\)
With probability at least $1-\beta_H(t)$, simultaneously for all responses,
\begin{equation}
|\widehat A_i^{\mathrm{credit}}-A_i^{\theta,\star}|
\leq
\frac{\min\{e_{R,i}+e_\theta(t),2\kappa\}}{s_{\min}}
+\frac{\kappa\bar e_s(t)}{s_{\min}^2}.
\label{eq:main_credit_stability}
\end{equation}
\end{theorem}

The credit bound includes the perturbed focal response. Its dependence on $e_{R,i}$ is capped by $2\kappa$, yielding magnitude-independent control under the stated geometry and population separation. Smaller $\kappa$ limits leverage more strongly, while larger $\kappa$ preserves moderate clean residuals more closely. Equation~\eqref{eq:credit_global_bound} gives the corresponding uniform advantage bound.

To characterize clean reference-estimation efficiency, write $\nu^2=\E\psi_c^2(X-\theta_\rho)/\{\E\psi_c'(X-\theta_\rho)\}^2$. For equal blocks of size $n$, assume independent honest block centers, $q_B=o(\sqrt B)$ arbitrary blocks, a $B^{-1/2}$-accurate initializer, an $o_{\Pp}(B^{-1/4})$ scale error, and the uniform local normal-score and equicontinuity conditions in \Cref{ass:transfer}.

\begin{theorem}[Clean outer efficiency]
\label{thm:main_clean_efficiency}
For fixed $K$, under the preceding transfer conditions,
\begin{equation}
\sqrt{Bn}(\widehat\theta_{\mathrm{RoVR}}-\theta_\rho)
\xrightarrow{d}\Normal(0,\nu^2V_K),\qquad
V_K:=\frac{\sum_{k,\ell}[\min(\tau_k,\tau_\ell)-\tau_k\tau_\ell]}
{\{\sum_k\phi(\Delta_k)\}^2}\longrightarrow\frac{\pi}{3}.
\label{eq:main_clean_efficiency}
\end{equation}
The limit first sends $(B,n)\to\infty$ at fixed $K$ and then refines the quantile grid.
\end{theorem}

The outer factor improves on the median value $\pi/2$ under the stated local normal-score model. Total reference-estimation variance combines this factor with the local sandwich factor and, for unequal blocks, the effective sample size in \eqref{eq:effective_sample_size}.

The ratio channel instead needs calibration and a clipping-margin statement. Let $|\widehat m-m^\star|\leq e_m$ and $(\widehat q,q^\star)=(e^{\widehat m},e^{m^\star})$.

\begin{proposition}[Calibration and clipping-margin control]
\label{prop:main_softrovr_stability}
At differentiable inputs, the converged SoftRoVR map, and any fixed number of equivariantly initialized solver steps, obey
\[
\SoftRoVR(\ell+a\mathbf1)=\SoftRoVR(\ell)+a,\quad
\sum_t\partial\SoftRoVR(\ell)/\partial\ell_t=1,\quad
e^{-e_m}\leq\widehat q/q^\star\leq e^{e_m}.
\]
A clipping branch can change only if
$\operatorname{dist}(m^\star,\{\log l,\log u\})\leq e_m$.
\end{proposition}

The proposition links log-weight error to the distance from a clipping boundary. \Cref{prop:policy_stability} extends this relation to same-branch objective and gradient perturbations under explicit Jacobian conditions.

\long\def\rovrDeferredTheory{
\section{Detailed Theoretical Statements}
\label{app:theory_statements}

This section connects estimator robustness to the group-relative update in four steps. We first give deterministic and fixed-group guarantees under mixed block and observation contamination. We then control the complete bounded credit and its target-fidelity cost. A clean transfer result separates the local sandwich variance from the outer factor and exploits scale orthogonality. The final results cover smoothing, clipping branches, linear leave-one-out baselines, and the bias introduced by robust credit shaping. Throughout, $\mathcal I(a,b)=[\min\{a,b\},\max\{a,b\}]$ denotes the closed interval between two real numbers.

\subsection{Mixed contamination at fixed group size}

Bounded score alone does not bound the displacement of a convex M-root because $\psi_c'$ vanishes in the tails. A deterministic comparison of two realized roots therefore requires curvature on the interval between them. This result is useful for token replacements, where no independence model is imposed.

\begin{proposition}
\label{prop:block_replacement}
Let datasets $D=(x_1,\ldots,x_n)$ and $D'=(x'_1,\ldots,x'_n)$ differ in at most $s$ coordinates, with convex M-centers $T,T'$. If
\[
\inf_{u\in\mathcal I(T,T')}
\frac1n\sum_{i=1}^{n}\psi_c'(x_i-u)\geq a_0>0,
\]
then
\begin{equation}
|T'-T|\leq\frac{2sM_c}{na_0}.
\label{eq:block_replacement_sensitivity}
\end{equation}
The same conclusion holds when the curvature condition is imposed on $D'$.
\end{proposition}

For reward groups, the main stochastic result avoids conditioning on this unknown root path. Let $g(t)=\E\psi_c(X-t)$ and suppose, for $0\leq u\leq r_0$,
\[
g(\theta+u)\leq-au,
\qquad
g(\theta-u)\geq au
\]
for $a>0$. In honest block $b$, let at most $s_b$ observations be replaced after drawing $n_b$ independent clean observations. Define
\begin{equation}
p_b(t)=
\exp\!\left[
-\frac{n_b}{2M_c^2}
\left(at-\frac{2s_bM_c}{n_b}\right)_+^2
\right].
\label{eq:block_mixed_tail}
\end{equation}
This bound follows by evaluating the contaminated score at the fixed threshold $\theta\pm t$; it does not require the empirical root to remain in a postulated curvature basin.

Random reward blocks only justify a within-block load statement if the partition is drawn independently of the corrupted set. The next proposition makes that scope explicit.

\begin{proposition}
\label{prop:random_partition}
Fix a set of $m$ corrupted indices among $N$ observations, and draw a balanced partition independently and uniformly at random. If $S_b$ is the number of corrupted indices in a block of size $n_b$, then for every $u>0$,
\begin{equation}
\Pp\!\left(
\max_{1\leq b\leq B}
\left\{S_b-\frac{n_bm}{N}\right\}\geq u
\right)
\leq\sum_{b=1}^{B}\exp\!\left(-\frac{2u^2}{n_b}\right)
\leq B\exp\!\left(-\frac{2u^2}{n_{\max}}\right).
\label{eq:random_partition_load}
\end{equation}
The statement does not apply to corruption selected after observing the partition.
\end{proposition}

We now combine non-identical honest-block tails with two distinct contamination units. Let $\mathcal A$ contain $q$ arbitrary blocks and let $\mathcal H$ be its complement, with $H=B-q$, $r_B=\lceil B/2\rceil$, and $q<r_B$. The set $\mathcal A$ must be fixed before the clean draws or sampled independently of them; values in these blocks may then be arbitrary and may depend on all clean observations. Latent clean observations in the honest blocks are mutually independent within and across blocks. After those observations are drawn, an adversary may choose the coordinates and values of at most $s_b$ replacements inside each honest block $b$. Thus the replacement mechanism may be adaptive to all clean values, but neither the honest/arbitrary block labels nor a partition invoked through \Cref{prop:random_partition} may be selected from those values. Define
\[
\overline p_H(t)=\frac1H\sum_{b\in\mathcal H}p_b(t),
\qquad
\eta_q=\frac{r_B-q}{B-q},
\]
and define the observable outer-correction magnitude and its deterministic envelope
\begin{equation}
\widehat C_B=
|\widehat\theta_{\mathrm{RoVR}}-\widehat\mu_0|,
\qquad
C_B=\frac{\nu_{\max}BK}{2D_KW_B},
\qquad \widehat C_B\leq C_B.
\label{eq:correction_magnitude}
\end{equation}
The value $\widehat C_B$ is computable from a realized sample, but it is not by itself a confidence radius. The probability statement below additionally fixes $t$ and depends on the population-separation quantities $(a,r_0)$ through $p_b(t)$; these quantities are not estimated post hoc from the same batch.

\begin{theorem}
\label{thm:finite_group}
Suppose the contamination timing and independence conditions above hold and every honest block satisfies the population score separation used in \eqref{eq:block_mixed_tail}. For $0<t\leq r_0$, if $\overline p_H(t)<\eta_q$, set
\[
\beta_H(t)=
\min\!\left\{1,\,2\exp\!\left[-2H
\{\eta_q-\overline p_H(t)\}^2\right]\right\}.
\]
Then the tie-neutral estimator in \eqref{eq:rovr} satisfies both
\begin{equation}
\Pp\!\left(
|\widehat\theta_{\mathrm{RoVR}}-\theta|>t+\widehat C_B
\right)
\leq\beta_H(t),
\qquad
\Pp\!\left(
|\widehat\theta_{\mathrm{RoVR}}-\theta|>t+C_B
\right)
\leq\beta_H(t).
\label{eq:finite_rovr_bound}
\end{equation}
\end{theorem}

\begin{corollary}
\label{cor:partition_to_estimator}
Under \Cref{thm:finite_group} with $q=0$, suppose the complete set of $m$ replacement indices is fixed before an independently sampled uniform balanced partition. For $0<\delta_{\rm part}<1$, use the block-specific deviation
\[
u_b(\delta_{\rm part})
=\sqrt{\frac{n_b}{2}\log\frac{B}{\delta_{\rm part}}},
\qquad
s_b(\delta_{\rm part})=
\min\!\left\{n_b,
\left\lceil\frac{n_bm}{N}+u_b(\delta_{\rm part})\right\rceil\right\}.
\]
Thus each block uses its own $n_b$ rather than a common deviation based on $n_{\max}$. Form $p_{b,\delta}(t)$ from \eqref{eq:block_mixed_tail}, and let
$\overline p_{B,\delta}(t)=B^{-1}\sum_bp_{b,\delta}(t)$ and
$\eta_0=\lceil B/2\rceil/B$. If
$\overline p_{B,\delta}(t)<\eta_0$, define
\[
\beta_{B,\delta}(t)=
\min\!\left\{1,\,
\delta_{\rm part}
+\min\!\left\{1,\,2\exp\!\left[-2B
\{\eta_0-\overline p_{B,\delta}(t)\}^2\right]\right\}
\right\}.
\]
Then, jointly over the clean sample and partition,
\begin{equation}
\Pp\!\left(
|\widehat\theta_{\mathrm{RoVR}}-\theta|>t+\widehat C_B
\right)
\leq\beta_{B,\delta}(t),
\qquad
\Pp\!\left(
|\widehat\theta_{\mathrm{RoVR}}-\theta|>t+C_B
\right)
\leq\beta_{B,\delta}(t).
\label{eq:partition_estimator_bound}
\end{equation}
The conclusion does not cover indices selected after the partition, a block designation selected by inspecting realized block values, or corruption that can observe the partition RNG before choosing its support. As in \Cref{thm:finite_group}, the conditional error statement still depends on the declared population-separation model.
\end{corollary}

The theorem covers fixed $B$ and balanced unequal block sizes under the declared population-separation model. The quantity $\widehat C_B$ measures the realized outer correction; an error radius $t+\widehat C_B$ additionally uses the ex ante threshold $t$ and separation parameters $(a,r_0)$. At $B=1$, the estimator equals the global M-center. Clean outer-efficiency rates are characterized separately in \Cref{sec:clean_efficiency}.

\subsection{Target, bounded response advantage, and retained contrast}

Robust location changes its population target under asymmetry, while bounded credit changes the reward utility. This subsection separates both effects from sampling and contamination error.

\begin{proposition}
\label{prop:asymmetric_bias}
Let $\mu=\E X$ be finite, $g(t)=\E\psi_c(X-t)$, and $\theta_\rho$ be a root of $g$. If
\[
\inf_{u\in\mathcal I(\mu,\theta_\rho)}
\E\psi_c'(X-u)\geq a_0>0,
\]
then
\begin{equation}
|\theta_\rho-\mu|
\leq\frac{|\E\psi_c(X-\mu)|}{a_0}.
\label{eq:asymmetric_functional_bias}
\end{equation}
If $\E|X-\mu|^3<\infty$ and
\[
\inf_{u\in\mathcal I(\mu,\theta_\rho)}
\E\!\left[\left\{1+\frac{(X-u)^2}{c^2}\right\}^{-3/2}\right]
\geq \bar a_0>0,
\]
then
\begin{equation}
|\theta_\rho-\mu|
\leq\frac{\E|X-\mu|^3}{2\bar a_0c^2}.
\label{eq:third_moment_bias}
\end{equation}
\end{proposition}

\begin{proposition}
\label{prop:population_contamination}
Let $P_\epsilon=(1-\epsilon)P_0+\epsilon Q$, where $P_0$ is symmetric about $\mu_0$. If the clean score obeys $|\E_{P_0}\psi_c(X-t)|\geq a_0|t-\mu_0|$ between $\mu_0$ and $\theta_\rho(P_\epsilon)$, then
\begin{equation}
|\theta_\rho(P_\epsilon)-\mu_0|
\leq\frac{\epsilon M_c}{(1-\epsilon)a_0}.
\label{eq:population_contamination_bias}
\end{equation}
\end{proposition}

The next result controls every response, including the replaced focal response omitted by a center-only analysis.

\begin{proposition}
\label{prop:credit_stability}
Let reward groups $R,R^\star\in\R^G$ differ in at most $r$ coordinates, and suppose their RoVR centers differ by at most $e_\theta$. Construct scales and credits with the same $\chi_\kappa$ as in \eqref{eq:credit_advantage}. Then
\begin{equation}
e_s:=|\widehat s_\chi-s_\chi^\star|
\leq
\frac{r\kappa^2/G+2\kappa e_\theta}{2s_{\min}}.
\label{eq:credit_scale_stability}
\end{equation}
For any index $i$, let $e_{R,i}=|R_i-R_i^\star|$ and $d_{\chi,i}=\min\{e_{R,i}+e_\theta,2\kappa\}$. Then
\begin{equation}
|\widehat A_i^{\mathrm{credit}}-A_i^{\mathrm{credit},\star}|
\leq
\frac{d_{\chi,i}}{s_{\min}}
+\frac{\kappa e_s}{s_{\min}^2},
\qquad
|\widehat A_i^{\mathrm{credit}}|\leq\frac{\kappa}{s_{\min}}.
\label{eq:credit_stability}
\end{equation}
The bound is independent of the magnitude of a replaced reward. For an unchanged response, $d_{\chi,i}\leq e_\theta$ and its sign is preserved whenever $|R_i-\theta_R^\star|>e_\theta$.
\end{proposition}

The center and credit bounds can be composed without hiding the focal replacement. Let $R^\star$ denote the clean precursor group and construct $A_i^{\theta,\star}$ from $R^\star$ using population reference $\theta$, the same $\chi_\kappa$, and the same $s_{\min}$.

\begin{corollary}
\label{cor:finite_credit}
Under \Cref{thm:finite_group}, suppose the observed and precursor reward groups differ in at most $r$ coordinates. Set
\[
e_\theta(t)=t+C_B,
\qquad
\overline e_s(t)=
\frac{r\kappa^2/G+2\kappa e_\theta(t)}{2s_{\min}},
\]
and $e_{R,i}=|R_i-R_i^\star|$. With probability at least $1-\beta_H(t)$, simultaneously for all responses,
\begin{equation}
|\widehat A_i^{\mathrm{credit}}-A_i^{\theta,\star}|
\leq
\frac{\min\{e_{R,i}+e_\theta(t),2\kappa\}}{s_{\min}}
+\frac{\kappa\overline e_s(t)}{s_{\min}^2}.
\label{eq:finite_credit_bound}
\end{equation}
The same population-separation-dependent probability statement holds with the observable correction magnitude
$e_\theta(t)=t+\widehat C_B$ and the corresponding
$\overline e_s(t)$. Hence the complete credit perturbation is independent of replacement magnitude, although it can still be numerically loose when $s_{\min}$ is small. A magnitude-independent policy-gradient statement additionally requires a score-Jacobian bound, a stable clipping branch, or explicit gradient clipping; \Cref{prop:policy_stability} states only the corresponding local conditions.
\end{corollary}

Boundedness need not erase moderate clean contrasts. If two clean residuals lie in $[-M,M]$, the mean-value theorem and $\chi_\kappa'(u)\geq m_{\kappa,M}:=(1+M^2/\kappa^2)^{-3/2}$ give
\begin{equation}
|\widehat A_i^{\mathrm{credit}}-\widehat A_j^{\mathrm{credit}}|
\geq
\frac{m_{\kappa,M}}{\sqrt{s_{\min}^2+\kappa^2}}
|R_i-R_j|.
\label{eq:credit_contrast_retention}
\end{equation}
Together, \eqref{eq:credit_clean_fidelity} and \eqref{eq:credit_contrast_retention} state the trade-off: larger $\kappa$ more closely preserves raw credit, while smaller $\kappa$ more strongly limits focal leverage.

\subsection{Clean efficiency and scale orthogonality}
\label{sec:clean_efficiency}

The finite-group bound is conditional on the declared population score separation and uses a deterministic envelope for the outer correction. Clean efficiency instead follows from a separate local block-center model. We first state an exact oracle calculation for unequal blocks, then give the transfer result for the implemented estimator.

Let
\[
a_\rho=\E\psi_c'(X-\theta_\rho),
\qquad
b_\rho=\E\psi_c^2(X-\theta_\rho),
\qquad
\nu^2=\frac{b_\rho}{a_\rho^2}.
\]
The clean transfer assumes $0<a_\rho<\infty$, $0<b_\rho<\infty$, $a_\rho>a_{\min}$, and $\nu\in(\nu_{\min},\nu_{\max})$, so the implemented floor and cap are asymptotically inactive. Let
\[
A_K=\sum_{k=1}^{K}\sum_{\ell=1}^{K}
[\min(\tau_k,\tau_\ell)-\tau_k\tau_\ell],
\qquad
V_K=\frac{A_K}{D_K^2}.
\]

\begin{proposition}
\label{prop:oracle_unequal_efficiency}
Suppose independent block centers satisfy $\widetilde\mu_b\sim\Normal(\theta,\nu^2/n_b)$ exactly. If the one-step correction in \eqref{eq:rovr} is evaluated at the oracle initializer $\theta$ and scale $\nu$, then it is unbiased and
\begin{equation}
\operatorname{Var}(\widehat\theta_{\mathrm{oracle}})
=\frac{\nu^2V_K}{N_{\mathrm{eff}}},
\qquad N_{\mathrm{eff}}=\frac{W_B^2}{B}.
\label{eq:oracle_unequal_variance}
\end{equation}
\end{proposition}

The symmetric quantile grid makes the one-step score locally insensitive to scale. Define
\[
M_\Phi(t,s)=\sum_{k=1}^{K}\{\Phi(t+s\Delta_k)-\tau_k\}.
\]

\begin{proposition}
\label{prop:scale_orthogonality}
At $(t,s)=(0,1)$,
\begin{equation}
M_\Phi(0,1)=0,
\qquad
\partial_tM_\Phi(0,1)=D_K,
\qquad
\partial_sM_\Phi(0,1)=\sum_{k=1}^{K}\Delta_k\phi(\Delta_k)=0.
\label{eq:scale_orthogonality}
\end{equation}
For fixed $K$ and $(t,s)$ in a neighborhood of $(0,1)$,
\begin{equation}
M_\Phi(t,s)=D_Kt+O\!\left(t^2+(s-1)^2\right).
\label{eq:orthogonal_expansion}
\end{equation}
\end{proposition}

Scale orthogonality removes the first-order scale term under the local normal-score model. The transfer below requires only $\widehat\nu/\nu-1=o_{\Pp}(B^{-1/4})$; a root-$B$ pilot is sufficient but not necessary. This remains a conditional asymptotic statement: the local normal-score approximation, stochastic equicontinuity, and pilot-scale rate are explicit assumptions rather than consequences of a pointwise block-center CLT. For a concise statement, take equal blocks of size $n$ and define
\[
Z_{b,n}=\frac{\sqrt n(\widetilde\mu_b-\theta_\rho)}{\nu},
\qquad F_n(z)=\E[J_0(Z_{b,n}-z)],
\]
\[
M_n(t,s)=\sum_{k=1}^{K}\{F_n(t+s\Delta_k)-\tau_k\},
\]
For the realized row, let $\mathcal H_B$ be the honest-block set and $H_B=|\mathcal H_B|=B-q_B$. Define the all-block composite score and its honest centered fluctuation by
\begin{align*}
\widehat M_{B,n}(t,s)
&=\frac1B\sum_{b=1}^{B}\sum_{k=1}^{K}
\{J_0(Z_{b,n}-t-s\Delta_k)-\tau_k\},\\
\mathbb U_{B,n}(t,s)
&=\frac1B\sum_{b\in\mathcal H_B}\sum_{k=1}^{K}
\{J_0(Z_{b,n}-t-s\Delta_k)-F_n(t+s\Delta_k)\}.
\end{align*}

\begin{assumption}
\label{ass:transfer}
For fixed $K$, honest block centers are independent and identically distributed within each row, while $q_B=o(\sqrt B)$ blocks may be arbitrary. The population quantities satisfy $0<a_\rho<\infty$, $0<b_\rho<\infty$, $a_\rho>a_{\min}$, and $\nu\in(\nu_{\min},\nu_{\max})$. With $t_0=\sqrt n(\widehat\mu_0-\theta_\rho)/\nu$ and $\widehat s=\widehat\nu/\nu$:
\begin{enumerate}
\item[(T1)] $t_0=O_{\Pp}(B^{-1/2})$.
\item[(T2)] $\widehat s-1=o_{\Pp}(B^{-1/4})$.
\item[(T3)] For every fixed $C$ and a deterministic $\varrho_B=o(B^{-1/4})$ containing $|\widehat s-1|$ with probability tending to one,
\[
\sup_{|t|\leq C/\sqrt B,\,|s-1|\leq \varrho_B}
|M_n(t,s)-M_\Phi(t,s)|=o(B^{-1/2}).
\]
\item[(T4)] On the same neighborhood,
\[
\sup|\mathbb U_{B,n}(t,s)-\mathbb U_{B,n}(0,1)|=o_{\Pp}(B^{-1/2}).
\]
\item[(T5)] The vector $\{J_0(Z_{b,n}-\Delta_k)\}_{k=1}^{K}$ has covariance converging to $\min(\tau_k,\tau_\ell)-\tau_k\tau_\ell$; in particular $F_n(\Delta_k)\to\tau_k$ and atoms at the thresholds vanish.
\end{enumerate}
\end{assumption}

Conditions (T3)--(T4) are local normal-score approximation and stochastic equicontinuity. They are not implied by a pointwise CLT, especially for tie-rich rule rewards. Smooth M-root linearization, a local density approximation, and a root-$B$ pilot scale are sufficient routes to these conditions; the fixed-group theorem remains the relevant result when they fail.

\begin{theorem}
\label{thm:efficiency_transfer}
Under \Cref{ass:transfer},
\begin{equation}
\sqrt{Bn}(\widehat\theta_{\mathrm{RoVR}}-\theta_\rho)
=-
\frac{\nu}{D_K\sqrt B}
\sum_{b\in\mathcal H_B}\sum_{k=1}^{K}
[J_0(Z_{b,n}-\Delta_k)-F_n(\Delta_k)]
+o_{\Pp}(1),
\label{eq:efficiency_linearization}
\end{equation}
and therefore
\[
\sqrt{Bn}(\widehat\theta_{\mathrm{RoVR}}-\theta_\rho)
\xrightarrow{d}\Normal(0,\nu^2V_K).
\]
For the sequential limit in which $(B,n)\to\infty$ first at fixed $K$ and then $K\to\infty$, $V_K\to\pi/3$.
\end{theorem}

The theorem isolates the outer asymptotic factor. Combining it with the local sandwich factor and the unequal-block effective sample size yields the complete asymptotic variance decomposition.

When $\theta_\rho=\mu$ and $0<\sigma^2<\infty$, clean efficiency relative to the sample mean is
\begin{equation}
\operatorname{ARE}_{\mathrm{mean}}
=\frac{\sigma^2}{\nu^2V_K},
\label{eq:total_are}
\end{equation}
not $1/V_K$. Under asymmetry the estimands differ, so this is only a variance ratio. For quadratic local centers under a Gaussian model, $\nu^2=\sigma^2$ and the outer limit is $3/\pi\approx0.955$; for the robust local center, the sandwich factor must be included.

\subsection{Smoothing, clipping branches, and policy scope}

For fixed block centers, initializer, and scale, let
\[
u_{bk}=\widetilde\mu_b-\widehat\mu_0-
\frac{\widehat\nu\Delta_k}{\sqrt{n_b}},
\qquad
m_\star=\min_{b,k:\,u_{bk}\neq0}|u_{bk}|.
\]
At a tie $u_{bk}=0$, both $H_\gamma$ and $J_0$ equal $1/2$, so that indicator term contributes zero discrepancy. Initializer, scale, and solver residuals of the full smooth map remain separate.

\begin{proposition}
\label{prop:smoothing}
If at least one $u_{bk}$ is nonzero, then
\begin{equation}
|\mathcal S_\gamma-\RoVR|
\leq
\frac{\widehat\nu BK}{D_KW_B}
\exp\!\left(-\frac{m_\star}{\gamma}\right).
\label{eq:smoothing_margin_bound}
\end{equation}
If each margin has no nonzero atoms and its continuous density is bounded by $L_z$, while $\widehat\nu\leq\nu_{\max}$, then $\E|\mathcal S_\gamma-\RoVR|\leq2\nu_{\max}BKL_z\gamma\log2/(D_KW_B)$. Atoms at zero contribute exactly zero under the mid-rank convention.
\end{proposition}

For block-center vectors $T,T'$ on equal blocks of size $n$ and pooled scales $0<\lambda,\lambda'\leq\lambda_{\max}$, let $d_T=\|T-T'\|_\infty$, $d_\lambda=|\lambda-\lambda'|$, and $\Delta_{\max}=\max_k|\Delta_k|$. With the raw-residual logistic parameterization $H_\gamma(u)$ defined above, the deterministic Lipschitz calculation in the appendix gives
\begin{equation}
\begin{split}
|\mathcal S_\gamma(T,\lambda)-\mathcal S_\gamma(T',\lambda')|
\leq{}&d_T+\frac{Kd_\lambda}{2\sqrt nD_K}\\
&+\frac{\lambda_{\max}K}{4\gamma\sqrt nD_K}
\left(2d_T+\frac{\Delta_{\max}d_\lambda}{\sqrt n}\right).
\end{split}
\label{eq:rovr_lipschitz}
\end{equation}
Together with \Cref{prop:block_replacement}, this is a distribution-free token replacement statement. The $1/\gamma$ term is the explicit robustness--smoothness trade-off.

The converged SoftRoVR map is translation equivariant. At differentiable inputs,
\begin{equation}
\SoftRoVR_{\gamma,\eta,S}(\ell+a\mathbf1)=\SoftRoVR_{\gamma,\eta,S}(\ell)+a,
\qquad
\sum_{t=1}^{T}\frac{\partial\SoftRoVR_{\gamma,\eta,S}(\ell)}{\partial\ell_t}=1.
\label{eq:rovr_translation_calibration}
\end{equation}
The same holds for a fixed number of equivariantly initialized solver steps. This common-shift calibration does not make the robust aggregate an exact likelihood ratio; asymmetry and position drift can still change its target.

Let ideal and observed log aggregates satisfy $|\widehat m-m^\star|\leq e_m$. Then
\begin{equation}
e^{-e_m}\leq\frac{\widehat q}{q^\star}\leq e^{e_m},
\qquad
\left|\frac{\widehat q}{q^\star}-1\right|\leq e^{e_m}-1.
\label{eq:sequence_weight_relative_error}
\end{equation}
Branch stability can be checked rather than assumed globally.

\begin{proposition}
\label{prop:branch_flip_count}
For sequence $i$, let $m_i^\star=\log q_i^\star$ and $|\widehat m_i-m_i^\star|\leq e_{m,i}$. A ratio clipping boundary can change only if
\[
\operatorname{dist}\!\left(m_i^\star,\{\log l,\log u\}\right)\leq e_{m,i}.
\]
Let $e_{R,i}=|R_i-R_i^\star|$. An advantage sign can change only if $|R_i^\star-\theta_R^\star|\leq e_{R,i}+e_\theta$; for an unchanged reward, $e_{R,i}=0$. Hence
\begin{equation}
N_{\mathrm{ratio\mbox{-}flip}}
\leq\sum_i\ind\!\left\{\operatorname{dist}(m_i^\star,\{\log l,\log u\})\leq e_{m,i}\right\},
\quad
N_{\mathrm{sign\mbox{-}flip}}
\leq\sum_i\ind\{|R_i^\star-\theta_R^\star|\leq e_{R,i}+e_\theta\}.
\label{eq:flip_count_bounds}
\end{equation}
\end{proposition}

Let $e_A=|\widehat A-A^\star|$, using \eqref{eq:credit_stability} for $\widehat A^{\mathrm{credit}}$. The matched clean reference here uses the same robust functionals; bias relative to the arithmetic-mean objective is a separate target/fidelity term.

\begin{proposition}
\label{prop:policy_stability}
Suppose $|A^\star|\leq A_{\max}$, $q^\star,\widehat q\leq Q_{\max}$, and $|m^\star|,|\widehat m|\leq L$. Set $\overline Q=\max\{Q_{\max},u\}$. Then
\begin{equation}
|g(\widehat q,\widehat A)-g(q^\star,A^\star)|
\leq\overline Q e_A+A_{\max}e^Le_m.
\label{eq:policy_value_stability}
\end{equation}
If both pairs have the same advantage sign, remain in the same differentiable clipping branch, $\|\nabla m^\star\|\leq J_{\max}$, and $\|\nabla\widehat m-\nabla m^\star\|\leq e_J$, then
\begin{equation}
\|\nabla g(\widehat q,\widehat A)-\nabla g(q^\star,A^\star)\|
\leq e^L\{e_A(J_{\max}+e_J)+A_{\max}(e_J+J_{\max}e_m)\}.
\label{eq:policy_gradient_stability}
\end{equation}
The same-branch premise is guaranteed whenever the margins in \Cref{prop:branch_flip_count} exceed their errors. No global policy-improvement claim follows at clipping kinks.
\end{proposition}

Finally, linear leave-one-out centering and bounded credit have different semantics.

\begin{theorem}
\label{thm:crossfit_baseline}
Fix a prompt $x$ and let $Y_1,\ldots,Y_G$ be independent samples from $\pi_\theta(\cdot\mid x)$. Suppose support is parameter independent and differentiation may pass through integration. If $B_i$ and $S_i>0$ are measurable with respect to $(x,Y_{-i},U)$ for independent auxiliary randomness $U$, assume $\E[S_i^{-1}\mid x]<\infty$ and
\[
\E\!\left[\|\nabla_\theta\log\pi_\theta(Y_i\mid x)\|
\frac{|R(x,Y_i)|+|B_i|}{S_i}\middle|x\right]<\infty.
\]
Then
\begin{equation}
\E\!\left[
\frac1G\sum_{i=1}^{G}
\nabla_\theta\log\pi_\theta(Y_i\mid x)
\frac{R(x,Y_i)-B_i}{S_i}
\middle|x\right]
=w(x)\nabla_\theta\E_{\pi_\theta}[R(x,Y)\mid x],
\label{eq:loo_scaled_direction}
\end{equation}
where $w(x)=G^{-1}\sum_i\E[S_i^{-1}\mid x]\in(0,\infty)$. If $S_i\equiv1$, then $w(x)=1$; a deterministic $S_i\equiv s$ instead gives $w(x)=1/s$. Exact off-policy likelihood ratios give the analogous identity under old-policy sampling, with the expectation defining $w$ taken under the old policy.
\end{theorem}

Equation \eqref{eq:loo_scaled_direction} preserves each prompt's gradient direction, but prompt-dependent $w(x)$ reweights the global objective. Clipping, the robust sequence surrogate, and bounded credit introduce additional changes.

\begin{proposition}
\label{prop:credit_gradient_distortion}
Under the conditions of \Cref{thm:crossfit_baseline}, set $U_i=R(x,Y_i)-B_i$ and suppose $S_i\geq s_{\min}$. At a common leave-one-out scale, the numerator-shaping difference between bounded-credit and linear score directions satisfies
\begin{equation}
\left\|
\E\!\left[
\nabla_\theta\log\pi_\theta(Y_i\mid x)
\frac{\chi_\kappa(U_i)-U_i}{S_i}
\middle|x\right]
\right\|
\leq
\frac{1}{2\kappa^2s_{\min}}
\E\!\left[\|\nabla_\theta\log\pi_\theta(Y_i\mid x)\|\,|U_i|^3\middle|x\right].
\label{eq:credit_gradient_distortion}
\end{equation}
\end{proposition}

The preceding bound isolates numerator shaping at a common scale. The full $\widehat A^{\mathrm{credit}}$ construction also replaces a raw scale by the mapped scale; algebraically,
\[
\left|\frac{\chi_\kappa(U)}{S_\chi}-\frac{U}{S_{\rm lin}}\right|
\leq\frac{|\chi_\kappa(U)-U|}{S_\chi}
+\frac{|U|\,|S_\chi-S_{\rm lin}|}{S_\chi S_{\rm lin}},
\]
so the scale-construction term must be bounded under additional moments or reported empirically. Thus the error to the original arithmetic objective decomposes into contamination error to the matched robust clean reference, robust-location target bias, numerator-shaping distortion, scale-construction distortion, and smoothing or solver error. Only the first component is addressed by a stability theorem; the others are deliberate surrogate choices.

}

\FloatBarrier
\section{Experiments}
\label{sec:experiments}

We evaluate RoVR-GSPO on mathematical reasoning, long-context summarization, and tool-call annotation, followed by channel ablations and controlled perturbation analyses. Detailed protocols and supplementary diagnostics appear in \Cref{app:additional_experiments,app:exp_protocol}.

\subsection{Experimental setup}
\label{sec:exp_protocol}

We evaluate three downstream tasks. For mathematical reasoning, Qwen3-4B-Base and OLMo-3-7B-Instruct are trained on the 7,500-example MATH split~\citep{hendrycks2021measuringmathematicalproblemsolving} and evaluated on MATH500, AIME2025, AMC23, Gaokao2023-Math-En~\citep{zhang2024evaluatingperformancelargelanguage}, and MinervaMath~\citep{lewkowycz2022solvingquantitativereasoningproblems}, using either a bounded answer verifier or InternLM2-1.8B-Reward. For GovReport summarization~\citep{huang2021efficient}, training uses LongDocFACTScore~\citep{bishop2024longdocfactscore} while evaluation uses ROUGE-1/2/L~\citep{lin2004rouge}; the evaluation metrics are therefore distinct from the training reward. For offline tool-call annotation, Qwen3-4B-Base is trained on the RLLA-4K split from ToolRL~\citep{qian2025toolrl}, and the held-out metric is the training scorer's format-plus-tool-match reward. No tools are executed in this third task.

GSPO~\citep{zheng2025gspo} is the parent baseline. Initialization, response and optimizer budgets, decoding, and evaluation code are matched within each comparison. RoVR-GSPO combines the full RoVR reward reference, bounded credit, and SoftRoVR sequence weights; channel ablations disable the corresponding component. The mathematical-reasoning main comparison and channel ablation report mean $\pm$ standard deviation over three independent runs; summarization and tool-call annotation are single-run estimates. \Cref{app:simulation} separately evaluates the reference estimator.

\subsection{Downstream results}
\label{sec:downstream_results}

\paragraph{Mathematical reasoning.}
\label{sec:math_results}

\begin{table}[!htbp]
\centering
\caption{Mathematical-reasoning accuracy (\%), mean $\pm$ standard deviation over three runs. Bold marks the higher mean for each setting.}
\label{tab:math_main}
\small
\setlength{\tabcolsep}{4pt}
\renewcommand{\arraystretch}{1.10}
\begin{tabular}{llcccc}
\toprule
\textbf{Model} & \textbf{Dataset}
& \multicolumn{2}{c}{\textbf{Rule reward}}
& \multicolumn{2}{c}{\textbf{InternLM2 reward}} \\
\cmidrule(lr){3-4}\cmidrule(lr){5-6}
& & \textbf{GSPO} & \textbf{RoVR-GSPO} 
& \textbf{GSPO} & \textbf{RoVR-GSPO} \\
\midrule

\multirow{5}{*}{Qwen3-4B-Base}
& MATH500      
& $77.00 \pm 0.43$ 
& $\mathbf{77.47 \pm 0.57}$ 
& $74.53 \pm 0.41$ 
& $\mathbf{75.80 \pm 0.28}$ \\

& Gaokao2023en 
& $39.74 \pm 0.21$ 
& $\mathbf{45.37 \pm 0.95}$ 
& $38.53 \pm 0.32$ 
& $\mathbf{42.08 \pm 0.37}$ \\

& MinervaMath  
& $19.36 \pm 0.46$ 
& $\mathbf{21.69 \pm 1.04}$ 
& $15.81 \pm 0.30$ 
& $\mathbf{16.67 \pm 0.34}$ \\

& AIME2025     
& $14.44 \pm 1.57$ 
& $\mathbf{15.56 \pm 1.57}$ 
& $8.89 \pm 1.57$ 
& $\mathbf{15.56 \pm 1.57}$ \\

& AMC23        
& $62.50 \pm 2.04$ 
& $\mathbf{63.33 \pm 1.18}$ 
& $56.67 \pm 1.18$ 
& $\mathbf{64.17 \pm 3.12}$ \\

\midrule

\multirow{5}{*}{OLMo-3-7B-Instruct}
& MATH500      
& $89.80 \pm 0.16$ 
& $\mathbf{90.73 \pm 0.25}$ 
& $88.93 \pm 0.74$ 
& $\mathbf{89.53 \pm 0.41}$ \\

& Gaokao2023en 
& $73.77 \pm 0.21$ 
& $\mathbf{74.64 \pm 0.74}$ 
& $71.94 \pm 0.36$ 
& $\mathbf{72.90 \pm 0.25}$ \\

& MinervaMath  
& $28.80 \pm 0.76$ 
& $\mathbf{29.17 \pm 0.34}$ 
& $28.56 \pm 0.17$ 
& $\mathbf{29.29 \pm 0.63}$ \\

& AIME2025     
& $37.78 \pm 1.57$ 
& $\mathbf{41.11 \pm 1.57}$ 
& $37.78 \pm 1.57$ 
& $\mathbf{41.11 \pm 3.14}$ \\

& AMC23        
& $86.67 \pm 2.36$ 
& $\mathbf{87.50 \pm 2.04}$ 
& $86.67 \pm 2.36$ 
& $\mathbf{90.83 \pm 1.18}$ \\

\bottomrule
\end{tabular}
\end{table}

RoVR-GSPO achieves a higher mean than matched GSPO in all $20$ settings. The five-benchmark average improvements are $2.08\%/3.97\%$ for Qwen3-4B-Base and $1.27\%/1.96\%$ for OLMo-3-7B-Instruct under rule/InternLM2 rewards.

\paragraph{Long-context summarization.}
\label{sec:summarization}

For GovReport, both backbones are fine-tuned with LongDocFACTScore rewards using eight sampled summaries per input and evaluated with ROUGE-1/2/L.

\begin{table}[!htbp]
\centering
\caption{GovReport ROUGE point estimates on a 0--1 scale; bold marks the higher score.}
\label{tab:govreport_rouge}
\setlength{\tabcolsep}{8pt}
\begin{tabular}{llccc}
\toprule
\textbf{Model} & \textbf{Method} & \textbf{ROUGE-1} & \textbf{ROUGE-2} & \textbf{ROUGE-L} \\
\midrule
\multirow{2}{*}{Qwen3-4B-Base}
& GSPO      & 0.309 & 0.142 & 0.166 \\
& RoVR-GSPO & \textbf{0.333} & \textbf{0.164} & \textbf{0.175} \\
\midrule
\multirow{2}{*}{OLMo-3-7B-Instruct}
& GSPO      & 0.552 & 0.189 & 0.221 \\
& RoVR-GSPO & \textbf{0.553} & \textbf{0.195} & \textbf{0.229} \\
\bottomrule
\end{tabular}
\end{table}

RoVR-GSPO improves ROUGE-1/2/L for both backbones: $0.024/0.022/0.009$ for Qwen3-4B-Base and $0.001/0.006/0.008$ for OLMo-3-7B-Instruct.

\long\def\rovrGovReportDynamicsFigure{
\begin{figure}[H]
\centering
\includegraphics[pagebox=mediabox,width=0.98\linewidth]{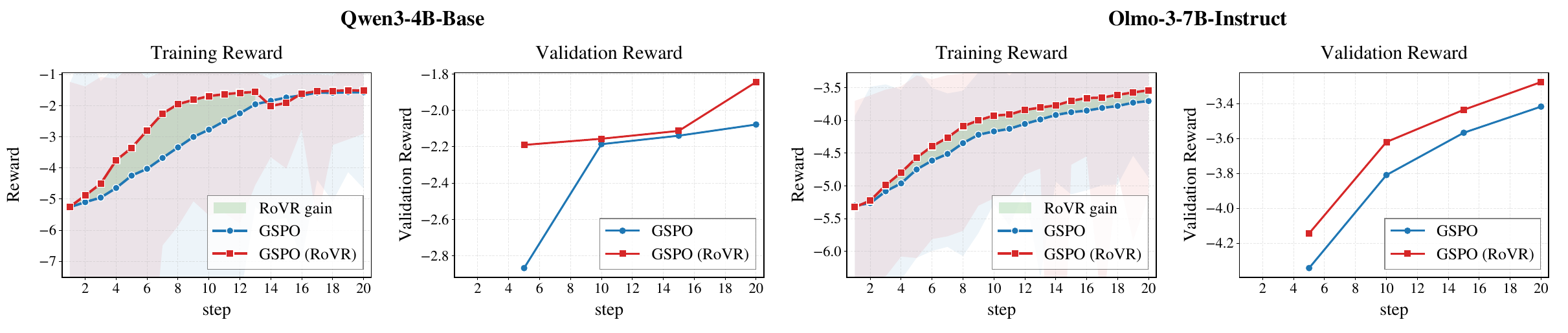}
\caption{GovReport training dynamics for Qwen3-4B-Base (left pair) and OLMo-3-7B-Instruct (right pair). Each pair shows training reward followed by validation reward. Blue circles denote GSPO, and red squares denote RoVR-GSPO, labeled GSPO (RoVR) in the figure. The green fill visualizes the reward gap.}
\label{fig:govreport_curves}
\end{figure}
}

The validation-reward trajectories in \Cref{fig:govreport_curves} are consistent with the held-out ROUGE results.

\paragraph{Tool-call annotation.}
\label{sec:tool_call_results}

On RLLA-4K, using 80 held-out prompts and one greedy response per prompt, the highest observed mean Total reward is $1.924$ for GSPO and $2.009$ for RoVR-GSPO. Each value is selected from four logged validations. The score measures format and tool matching without executing tools, and the available logs cover only the early portion of the configured training budget.

\subsection{Channel placement ablation}
\label{sec:robust_scale_exp}

We compare GSPO, reward-only RoVR-Adv, ratio-only SoftRoVR, and the full RoVR-GSPO model (\Cref{tab:placement_ablation,fig:placement_ablation}).


\begin{table}[!htbp]
	\centering
	\caption{Channel-ablation accuracy (\%) on mathematical reasoning. Values are reported as mean $\pm$ standard deviation over three independent runs; bold marks the highest mean in each column.}
	\label{tab:placement_ablation}
	\small
	\setlength{\tabcolsep}{5pt}
	\renewcommand{\arraystretch}{1.05}
	\begin{tabular}{lccccc}
		\toprule
		\textbf{Method} & \textbf{MATH500} & \textbf{Gaokao} & \textbf{Minerva} & \textbf{AIME25} & \textbf{AMC23} \\
		\midrule
		GSPO
		& $74.53 \pm 0.50$
		& $38.53 \pm 0.40$
		& $15.81 \pm 0.37$
		& $8.89 \pm 1.92$
		& $56.67 \pm 1.44$ \\
		
		RoVR-Adv
		& $74.73 \pm 0.81$
		& $40.61 \pm 0.40$
		& $15.18 \pm 0.23$
		& $12.22 \pm 1.92$
		& $56.67 \pm 1.44$ \\
		
		SoftRoVR
		& $74.53 \pm 0.23$
		& $40.87 \pm 1.28$
		& $16.42 \pm 0.34$
		& $14.44 \pm 1.93$
		& $59.17 \pm 1.44$ \\
		
		RoVR-GSPO
		& $\mathbf{75.80 \pm 0.35}$
		& $\mathbf{42.08 \pm 0.45}$
		& $\mathbf{16.67 \pm 0.42}$
		& $\mathbf{15.56 \pm 1.93}$
		& $\mathbf{64.17 \pm 3.82}$ \\
		\bottomrule
	\end{tabular}
\end{table}

Across the five benchmarks, the macro-average gains over GSPO are $1.00$, $2.20$, and $3.97$ percentage points for RoVR-Adv, SoftRoVR, and RoVR-GSPO, respectively, based on the three-run means. RoVR-GSPO has the highest mean on all five benchmarks and exceeds SoftRoVR by $1.77$ points on average.

\long\def\rovrAblationDynamicsFigure{
\begin{figure}[!htbp]
\centering
\includegraphics[width=\linewidth]{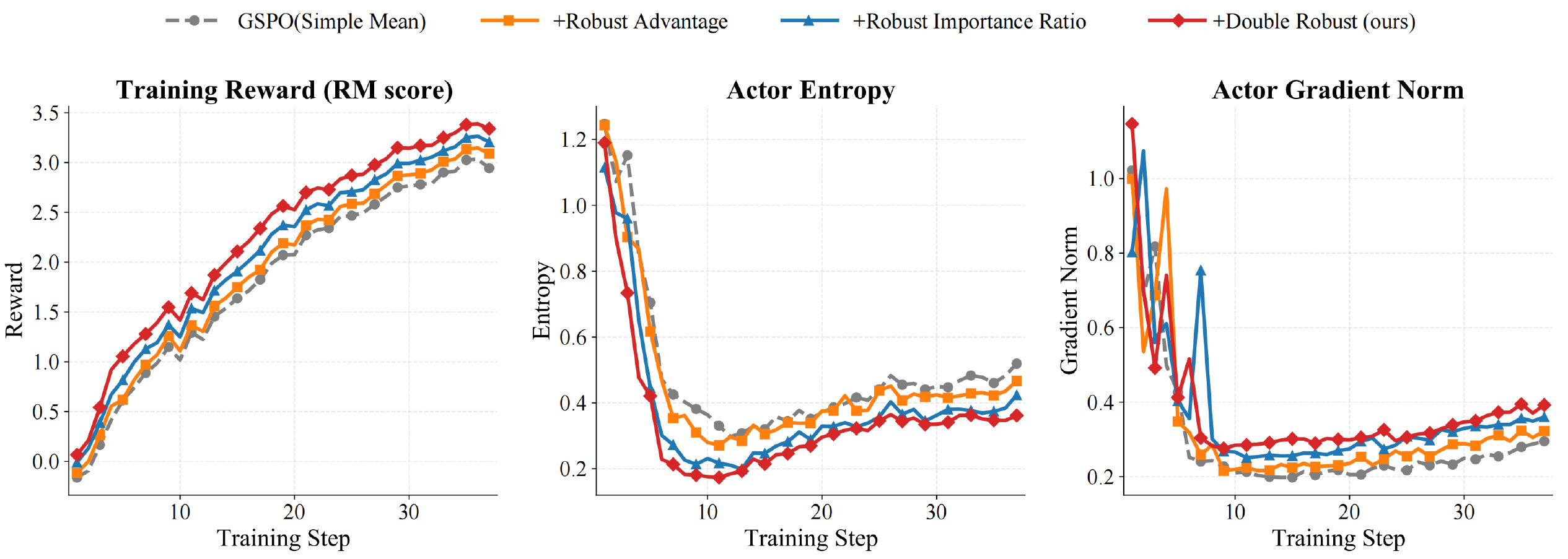}
\caption{Training dynamics for the channel ablation: reward, actor entropy, and actor gradient norm. Orange, blue, and red denote RoVR-Adv, SoftRoVR, and RoVR-GSPO; dashed gray denotes GSPO.}
\label{fig:placement_ablation}
\end{figure}
}

The combined variant records the highest late-stage reward, while the robust variants show lower late-stage entropy (\Cref{fig:placement_ablation}).

\paragraph{Mechanism checks.}
Offline perturbations support both mechanisms. Robust centering reduces reference displacement from $1.000\sigma$ to $0.060\sigma$, while bounded credit limits scale inflation from $28.194\times$ to $1.911\times$ and raises clean-contrast retention from $0.035$ to $0.520$. SoftRoVR reduces log-weight displacement by about $26\%$ for strong token spikes and $87\%$ for $20\%$ bursts, with fewer clipping flips under strong spikes.

\long\def\rovrRobustnessExperiments{
\subsection{Localized token-ratio stress test}
\label{sec:softrovr_token_stress}

We evaluate aggregation sensitivity on saved Qwen3-4B-Base records from a dual-channel run using $\widehat A^{\mathrm{credit}}$ and SoftRoVR. No policy is retrained: GSPO applies arithmetic-mean aggregation to the recorded token log-ratios, while SoftRoVR uses the run's differentiable aggregation. We then measure how localized token-ratio perturbations change sequence log-weights and clipping branches.

The saved snapshot contains $35$ complete training steps, each with $192$ prompt groups and $16$ responses ($107{,}520$ responses in total). The partially recorded step $36$ contains only $1{,}408$ responses and is excluded from the stress analysis. We select step $18$ using a fixed upper-middle rule on the ordered complete steps and use all $3{,}072$ responses, including all eight within-step optimizer-update positions. For response $i$, let $\boldsymbol{\ell}_i$ be the valid token log-ratios and let $\sigma_i$ be their population standard deviation. A single-token spike adds $\Delta=\pm c\sigma_i$ at a uniformly sampled valid position, with $c\in\{0,2,4,8,16\}$. A contiguous burst adds $\pm8\sigma_i$ to $\lceil fT_i\rceil$ tokens, with $f\in\{0,0.01,0.05,0.10,0.20\}$. Each response uses five paired random positions across methods and severities. We average positions within a response, then responses within a prompt, and finally give every prompt equal weight. Pointwise intervals are $5{,}000$ prompt-cluster percentile bootstrap replicates.

We measure the absolute log-weight displacement
\[
D_q=\left|m_{\mathrm{corrupt}}-m_{\mathrm{clean}}\right|
\]
and the fraction of responses whose clipping branch changes relative to its own clean branch. Branches are evaluated with the recorded advantage and the actual GSPO clipping interval $[0.9997,1.0004]$. The zero-scale responses ($12.5\%$ of responses in the complete steps) remain in the primary analysis; their perturbation is exactly zero. The paired differences in \Cref{tab:softrovr_local_stress} are GSPO minus SoftRoVR, so positive values favor SoftRoVR.

\begin{figure}[!htbp]
\centering
\includegraphics[width=0.98\linewidth]{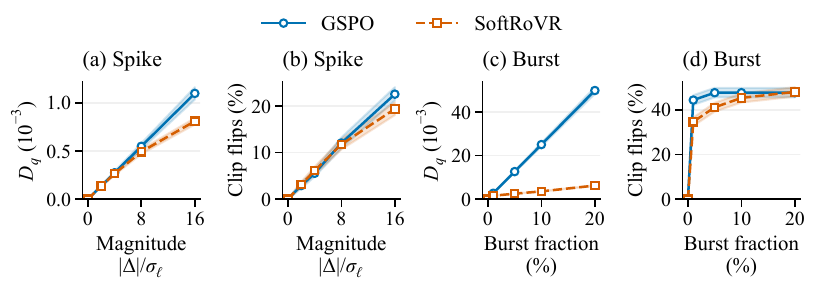}
\caption{Localized token-ratio stress test at recorded step $18$. Curves show means over $192$ prompt groups ($3{,}072$ responses); shaded regions show pointwise 95\% prompt-cluster bootstrap intervals. SoftRoVR attenuates log-weight displacement for both isolated spikes and contiguous bursts. For strong spikes, the paired flip-rate difference favors SoftRoVR; for the positive $20\%$ burst, its interval includes zero.}
\label{fig:softrovr_spike_burst}
\end{figure}

\begin{table}[!htbp]
\centering
\small
\setlength{\tabcolsep}{3pt}
\caption{Paired offline token-ratio stress test at step $18$ using $192$ prompt groups and $3{,}072$ responses. $D_q$ is reported in $10^{-3}$ units; for Flips, differences are reported in \%. Intervals are pointwise 95\% prompt-cluster bootstrap intervals.}
\label{tab:softrovr_local_stress}
\begin{tabular}{lllrrr}
\toprule
Perturbation & Sign & Metric & GSPO & SoftRoVR & Difference [95\% CI] \\
\midrule
Spike $16\sigma_\ell$ & $+$ & $D_q$ & 1.097 & 0.808 & 0.288 [0.256, 0.324] \\
Spike $16\sigma_\ell$ & $+$ & Flips (\%) & 22.526 & 19.362 & 3.164 [2.246, 4.082] \\
Burst $20\%$ & $+$ & $D_q$ & 49.812 & 6.258 & 43.554 [42.112, 44.963] \\
Burst $20\%$ & $+$ & Flips (\%) & 47.591 & 47.956 & -0.365 [-0.944, 0.221] \\
Spike $16\sigma_\ell$ & $-$ & $D_q$ & 1.097 & 0.811 & 0.286 [0.254, 0.321] \\
Spike $16\sigma_\ell$ & $-$ & Flips (\%) & 21.061 & 17.370 & 3.691 [2.949, 4.453] \\
Burst $20\%$ & $-$ & $D_q$ & 49.812 & 6.283 & 43.529 [42.088, 44.940] \\
Burst $20\%$ & $-$ & Flips (\%) & 39.909 & 39.102 & 0.807 [0.202, 1.413] \\
\bottomrule
\end{tabular}
\end{table}

At the strongest single-token spike, SoftRoVR reduces $D_q$ by $26.30\%$ for positive perturbations and $26.09\%$ for negative perturbations; clipping-flip rates decrease from $22.526\%$ to $19.362\%$ and from $21.061\%$ to $17.370\%$. For a $20\%$ burst, $D_q$ decreases by $87.44\%$ and $87.39\%$ in the positive and negative directions. The positive-burst flip rate changes from $47.591\%$ to $47.956\%$ with interval $[-0.944\%,0.221\%]$, while the negative-burst rate changes from $39.909\%$ to $39.102\%$ with interval $[0.202\%,1.413\%]$. Overall, SoftRoVR reduces aggregate displacement and strong-spike clipping sensitivity, while long-burst clipping effects depend on perturbation direction and training stage.

\Cref{app:softrovr_stress} provides the negative-direction and across-step analyses.

\subsection{Robustness to reward contamination}
\label{sec:reward_contamination}

We perturb saved InternLM2 rewards from Qwen3-4B-Base rollouts while keeping responses fixed. From $16{,}704$ complete groups, we select $5{,}000$ evenly spaced groups with $G=16$ and replace one reward at a time by
\begin{equation}
R_t' = R_t \mathbin{\pm} \alpha\sigma,
\qquad
\alpha\in\{0.5,1,2,4,8,16\},
\end{equation}
where $\sigma=1.165$ is the population standard deviation of the saved reward stream. We compare GRPO/GSPO normalization, the RoVR-reference variant $\widehat A^{\mathrm{center}}$, and the bounded-credit variant $\widehat A^{\mathrm{credit}}$. Both robust variants use the full RoVR reference, including blockwise M-centers with $c=1$ and the outer composite-quantile correction. The credit map uses $\kappa=1$ and $s_{\min}=10^{-3}$. Keeping the reference and block assignment matched isolates the contribution of the bounded residual map.

The baseline divides mean-centered rewards by the sample standard deviation (denominator $G-1$) plus $10^{-6}$. The RoVR-reference and bounded-credit variants use the RMS of their respective residuals with $10^{-6}$ inside the square root. At $G=16$ and negligible regularization, sample standard deviation is $\sqrt{16/15}$ times mean-centered RMS. The center-to-credit comparison holds both the RoVR reference and the RMS convention fixed.

We measure reference displacement, scale inflation, clean-response contrast retention $\mathcal C_H$, and clean-advantage RMS deviation relative to each estimator's clean outputs. Results use within-group medians, averaged over perturbation signs, followed by medians across prompt groups. Pointwise $95\%$ intervals use $500$ group-bootstrap replicates.

\begin{figure}[!htb]
\centering
\includegraphics[width=\linewidth]{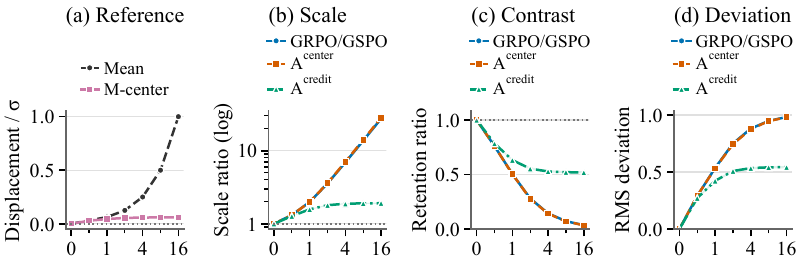}
\caption{Offline sensitivity to additive reward contamination. One of $G=16$ rewards is perturbed by $\pm\alpha\sigma$; curves summarize $5{,}000$ prompt groups and shaded regions denote bootstrap $95\%$ intervals. The RoVR reference (labeled M-center) limits location displacement, while bounded-credit normalization additionally limits scale inflation and preserves more clean-response advantage contrast.}
\label{fig:reward_contamination}
\end{figure}

\begin{table}[!htbp]
\centering
\caption{Estimator sensitivity at the strongest offline stress level ($\alpha=16$). Scale inflation and clean contrast retention are ideal at $1$; advantage RMS deviation is ideal at $0$. Values are medians across the same $5{,}000$ complete prompt groups used in \Cref{fig:reward_contamination}.}
\label{tab:reward_contamination}
\small
\setlength{\tabcolsep}{4pt}
\renewcommand{\arraystretch}{1.08}
\begin{tabular}{lccc}
\toprule
\textbf{Estimator} & \textbf{Scale inflation} & \textbf{Clean contrast retention} & \textbf{Advantage RMS deviation} \\
\midrule
GRPO/GSPO & 27.400 & 0.036 & 0.982 \\
Robust center & 28.194 & 0.035 & 0.980 \\
$A^{\mathrm{credit}}$ & \textbf{1.911} & \textbf{0.520} & \textbf{0.543} \\
\bottomrule
\end{tabular}
\end{table}

\Cref{fig:reward_contamination,tab:reward_contamination} separate protection of the reference from protection of the complete advantage. At $16\sigma$, the arithmetic reference shifts by $1.000\sigma$, whereas the robust reference shifts by only $0.060\sigma$. With the reference held fixed, the bounded-credit map further limits scale inflation to $1.911\times$, retains $0.520$ of the clean contrast, and reduces clean-advantage RMS deviation to $0.543$; the center-only variant retains only $0.035$ of the clean contrast and inflates the scale by $28.194\times$, close to standard normalization ($0.036$ and $27.400\times$).

At the tested $4\sigma$ shift, bounded-credit normalization yields $1.881\times$ scale inflation and $0.529$ contrast retention, compared with $6.918\times$ and $0.145$ for standard normalization. These fixed-batch results match the mechanism in \Cref{prop:center_only_impossibility}: robust centering controls shared reference displacement, while the bounded residual map additionally controls the influence of the perturbed response on scale and advantage. Complementary simulations in \Cref{app:simulation} isolate the clean outer-efficiency result and compare reference estimators under heavy tails and matched contamination geometries.

\FloatBarrier
}

\section{Conclusion}
\label{sec:conclusion}

RoVR-GSPO treats reward-derived advantages and sequence-level weights as separate statistical interfaces to group-relative policy optimization. Blockwise M-estimation and tie-neutral composite-quantile aggregation construct the reward reference; bounded credit protects the normalized response contrast, and SoftRoVR aggregates token log-ratios before clipping. The analysis links these design choices to reference robustness, clean outer efficiency, and local update stability.

Across mathematical reasoning and GovReport, RoVR-GSPO consistently improves over matched GSPO baselines, with additional gains on the offline tool-call annotation score. Controlled reward and token perturbations further show that the two channels address complementary sources of update sensitivity. These results support robust advantage and sequence-weight estimation as a practical statistical interface for group-relative language-model post-training.

\subsection*{AI use statement}

OpenAI Codex was used to assist with manuscript organization, language editing, LaTeX formatting, and stress-testing the exposition and logical consistency of the motivation and theoretical arguments. The authors are responsible for independently verifying every mathematical claim, proof, citation, experimental datum, and conclusion before submission, and they take responsibility for the final content of the paper.

\subsection*{Ethics statement}

This work studies robust advantage and sequence-weight estimation for language-model post-training and introduces no new human-subject study or personally identifiable dataset. Robust aggregation can reduce sensitivity to isolated scoring errors, but it does not guarantee reward-model validity, factual correctness, fairness, or safe deployment. Any deployment should therefore retain independent reward audits, distribution-shift evaluation, and the data-governance requirements of the underlying models and datasets.

\subsection*{Reproducibility statement}

The estimator, credit map, token extension, and optimization procedure are specified in \Cref{sec:method,app:implementation}. The assumptions and complete proofs of the theoretical claims appear in \Cref{app:finite_group,app:target_scale,app:policy}. The experimental protocol, placement ablation, and reporting conventions are specified in \Cref{sec:experiments,app:exp_protocol}.

\bibliography{reference}
\bibliographystyle{iclr2027_conference}

\clearpage
\appendix
\raggedbottom
\setlength{\textfloatsep}{5pt plus 1pt minus 1pt}
\setlength{\floatsep}{4pt plus 1pt minus 1pt}
\setlength{\intextsep}{5pt plus 1pt minus 1pt}
\renewcommand{\floatpagefraction}{0.8}
\makeatletter
\setlength{\@fptop}{0pt}
\setlength{\@fpsep}{8pt plus 2pt minus 2pt}
\setlength{\@fpbot}{0pt plus 1fil}
\makeatother

\section{Related Work}
\label{sec:related}

RoVR-GSPO connects two strands of work: robust construction of reward-derived advantages and robust aggregation of the likelihood statistics used by clipped policy objectives. PPO clips likelihood ratios and is commonly paired with a learned value function in RLHF~\citep{schulman2017ppo,ouyang2022instructgpt}; GRPO replaces the critic with within-group reward normalization~\citep{shao2024deepseekmath}; DAPO studies large-scale choices such as asymmetric clipping and dynamic sampling~\citep{yu2025dapo}; and GSPO uses a length-normalized sequence correction~\citep{zheng2025gspo}. RoVR-GSPO keeps the parent clipped objective and contributes a statistical interface that separates reference estimation, bounded credit, and sequence-weight aggregation.

\paragraph{Robust reward and advantage estimation.}
Learned reward models can exhibit label errors, calibration shifts, and exploitable tails, while outcome verifiers are bounded and often tie-rich. Process supervision supplies denser feedback~\citep{lightman2023verify}, and heavy-tailed reward misspecification remains a concern even with KL regularization~\citep{kwa2024catastrophic}. Reward winsorization and post-normalization advantage clipping limit extremes; RoVR separates robust reference estimation from bounded residual normalization so their effects can be analyzed independently. Risk-sensitive objectives intentionally change the optimized reward functional~\citep{ren2025riskpo}, whereas our reward channel exposes the robustness--fidelity trade-off through the explicit map $\chi_\kappa$.

Group-relative normalization also induces prompt-dependent difficulty weighting~\citep{yang2026grouprelativeadvantagebiased}. This statistical bias concerns the target of advantage estimation, whereas RoVR characterizes sensitivity to localized reward perturbations and the fidelity of bounded credit. The leave-one-out analysis in \Cref{thm:crossfit_baseline} makes the prompt-reweighting effect explicit.

\paragraph{Robust reference estimation.}
Medians, trimming, and bounded-score M-estimators provide classical contamination protection~\citep{huber2011robust,rieder2012robust}. A global Huber or pseudo-Huber estimator is the $B=1$ boundary of our reference construction. MOM methods add an explicit block interface~\citep{lecue2020robust,humbert2022robust}, while VRMOM reduces the outer median variance through a composite-quantile one-step correction but retains arithmetic local means~\citep{tu2021variancereducedmedianofmeansestimator}. RoVR combines bounded-score local centers with tie-neutral composite-quantile aggregation and then adds a policy-specific bounded credit map, so reference protection and clean-response contrast can be examined separately.

\paragraph{Robust sequence-weight estimation.}
GSPO summarizes token log-ratios into a sequence-level weight before clipping. SoftRoVR targets this aggregation step directly: it smooths the RoVR construction so gradients can pass through the sequence statistic and then exponentiates the resulting log-weight. Because this robust aggregation changes the likelihood functional, we analyze it as an optimization surrogate and study its calibration and clipping-branch stability rather than treating it as an exact importance ratio.

Barron's family connects several classical robust penalties~\citep{BarronCVPR2019} and supports adaptive and graduated robust estimation~\citep{jung2024adaptive,hitchcox2022mind}. RoVR fixes $\alpha=1$ to obtain a convex pseudo-Huber/Charbonnier-type local loss, then combines its bounded score with tie-neutral composite-quantile aggregation and policy-specific bounded residual normalization. This choice gives a unique local root and the score properties used in our analysis.

\section{Additional Experimental Results}
\label{app:additional_experiments}

This appendix reports training dynamics and the complete token-ratio and reward-contamination diagnostics using the sampling units specified in each protocol.

\subsection{Training dynamics}

\rovrGovReportDynamicsFigure
\rovrAblationDynamicsFigure

\rovrRobustnessExperiments

\long\def\rovrTheoryAppendix{
\rovrDeferredTheory

\section{Proofs for Failure Propagation and Finite-Group Robustness}
\label{app:finite_group}

This appendix proves the exact GRPO phase diagram, the translation, constant-preservation, and tie-neutral properties of RoVR, and the mixed-contamination theorem. The fixed-group proof evaluates bounded empirical scores at deterministic thresholds, thereby avoiding a curvature event along an unknown contaminated root path.

\subsection{Single-reward propagation and the center-only impossibility}

\begin{proof}[Proof of \Cref{prop:grpo_broadcast}]
Write $R_i^{(\Delta)}=R_i+\Delta\ind\{i=j\}$. The mean identity is immediate. Expanding the centered sum of squares, or applying the one-point variance update, gives
\[
\frac1G\sum_i(R_i^{(\Delta)}-\bar R_\Delta)^2
=\frac1G\sum_i(R_i-\bar R)^2
+\frac{2\Delta(R_j-\bar R)}{G}
+\frac{G-1}{G^2}\Delta^2,
\]
which proves \eqref{eq:grpo_spike_exact}, including the unchanged floor $\varepsilon_s$.

The centered residuals sum to zero, so $\sum_iA_{i,\Delta}=0$. Moreover,
\[
\sum_iA_{i,\Delta}^2
=\frac{\sum_i(R_i^{(\Delta)}-\bar R_\Delta)^2}
{G^{-1}\sum_i(R_i^{(\Delta)}-\bar R_\Delta)^2+\varepsilon_s}
\leq G.
\]
If $M=\max_i|A_{i,\Delta}|$, the other $G-1$ values sum to a number of magnitude $M$. Cauchy--Schwarz gives $\sum_{k\neq i}A_{k,\Delta}^2\geq M^2/(G-1)$, hence $G\geq M^2G/(G-1)$ and $M\leq\sqrt{G-1}$.

As $|\Delta|\to\infty$, \eqref{eq:grpo_spike_exact} gives $s_\Delta\sim|\Delta|\sqrt{G-1}/G$. The focal centered residual is $\Delta(G-1)/G+O(1)$ and every other centered residual is $-\Delta/G+O(1)$. Division proves \eqref{eq:grpo_broadcast_limit}. Finally, $A_{i,\Delta}-A_{k,\Delta}=(R_i-R_k)/s_\Delta\to0$ for clean $i,k$.
\end{proof}

\begin{proof}[Proof of \Cref{prop:center_only_impossibility}]
Because $m_\Delta=O(1)$, the focal raw residual is $R_j+\Delta-m_\Delta=\Delta+O(1)$, while every clean residual is $O(1)$. If $d_\Delta=O(1)$, the focal quotient diverges. Conversely, bounded focal credit requires $d_\Delta=\Omega(|\Delta|)$, and division sends every fixed clean residual to zero. For an RMS around $m_\Delta$, the single squared focal residual dominates its average, so $d_\Delta\sim|\Delta|/\sqrt G$ and the stated limits follow.
\end{proof}

\subsection{Invariance and deterministic root sensitivity}

\begin{proof}[Proof of \Cref{prop:calibration}]
Adding $a$ to every observation shifts each convex block objective by the change of variable $u\mapsto u+a$. Hence every $\widetilde\mu_b$ and their median shift by $a$, while all residuals, local sandwich scales, and the arguments of $J_0$ remain unchanged. Equation \eqref{eq:rovr} therefore shifts by exactly $a$.

If every observation equals $a$, each block center and the median equal $a$. Symmetry gives $\Delta_{K+1-k}=-\Delta_k$. For every noncentral pair,
\[
J_0(-\widehat\nu\Delta_k/\sqrt{n_b})
+J_0(-\widehat\nu\Delta_{K+1-k}/\sqrt{n_b})=1.
\]
When $K$ is odd, the central term is $J_0(0)=1/2$. Thus the sum over $k$ is $K/2=\sum_k\tau_k$ for every block, and the correction is zero. The same pairing applies when $B=1$ because its only block center equals the initializer, so RoVR reduces to that local M-center.
\end{proof}

\begin{proof}[Proof of \Cref{prop:block_replacement}]
Replacing at most $s$ observations changes the empirical score uniformly by at most $2sM_c/n$. Since the score derivative is the negative empirical curvature, the mean-value theorem between the two roots gives
\[
a_0|T'-T|
\leq\sup_u|S_D(u)-S_{D'}(u)|
\leq\frac{2sM_c}{n},
\]
which proves \eqref{eq:block_replacement_sensitivity}. Interchanging $D,D'$ gives the alternative condition.
\end{proof}

\subsection{Randomized load and the fixed-threshold block tail}

\begin{proof}[Proof of \Cref{prop:random_partition}]
For a fixed block of size $n_b$, $S_b$ is hypergeometric with mean $n_bm/N$. Hoeffding's inequality for sampling without replacement yields
\[
\Pp\!\left(S_b-\frac{n_bm}{N}\geq u\right)
\leq\exp(-2u^2/n_b).
\]
A union bound over blocks proves the first inequality in \eqref{eq:random_partition_load}. Since $n_b\leq n_{\max}$, each summand is at most $\exp(-2u^2/n_{\max})$, proving the second.
\end{proof}

\begin{proof}[Proof of \Cref{cor:partition_to_estimator}]
Apply \Cref{prop:random_partition} using the block-specific threshold for each block $b$,
$u_b=\sqrt{n_b\log(B/\delta_{\rm part})/2}$, and then take a union bound. With probability at least $1-\delta_{\rm part}$, every realized load is at most $s_b(\delta_{\rm part})$. Conditional on the sampled partition and this load event, the clean block samples remain independent, the replacement support obeys the deterministic budgets, and \Cref{thm:finite_group} applies with $q=0$. Its observable-correction-magnitude and deterministic-envelope statements give the two conditional bounds. A final union bound over the load event proves \eqref{eq:partition_estimator_bound}. If the replacement indices were chosen after seeing the partition, this conditioning step would not supply the prescribed independent load event.
\end{proof}

Let $S_b^0(v)=n_b^{-1}\sum_i\psi_c(X_{bi}-v)$ be the clean score and $S_b(v)$ the score after at most $s_b$ replacements. Uniformly in $v$,
\[
|S_b(v)-S_b^0(v)|\leq\frac{2s_bM_c}{n_b}.
\]
Because $S_b$ is decreasing, $\{T_b>\theta+t\}\subseteq\{S_b(\theta+t)>0\}$. Therefore
\[
\{T_b>\theta+t\}
\subseteq
\left\{
S_b^0(\theta+t)-\E S_b^0(\theta+t)
>at-\frac{2s_bM_c}{n_b}
\right\}.
\]
Each score is in $[-M_c,M_c]$, so Hoeffding's inequality gives the upper-tail probability $p_b(t)$ in \eqref{eq:block_mixed_tail}; the lower tail is identical.

\begin{proof}[Proof of \Cref{thm:finite_group}]
For every honest block define the clean upper-tail event
\[
E_b^+(t)=\left\{S_b^0(\theta+t)-\E S_b^0(\theta+t)
>at-\frac{2s_bM_c}{n_b}\right\}.
\]
The fixed-threshold inclusion gives
$\{T_b>\theta+t\}\subseteq E_b^+(t)$ regardless of how the allowed replacement coordinates and values are selected after the clean sample is observed. The events $\{E_b^+(t):b\in\mathcal H\}$ depend only on mutually independent latent clean blocks, so they remain independent and have probabilities at most $p_b(t)$. This is why adaptivity of replacement values is allowed but data-dependent selection of $\mathcal H$ is not.

If the midpoint median for even $B$, or the central median for odd $B$, exceeds $\theta+t$, at least $r_B=\lceil B/2\rceil$ block centers lie above the threshold. Even if all $q$ arbitrary blocks do so, at least $r_B-q$ clean dominating events must occur. Hoeffding's inequality for independent, non-identically distributed Bernoulli variables gives
\[
\Pp(\widehat\mu_0>\theta+t)
\leq\exp[-2H\{\eta_q-\overline p_H(t)\}^2].
\]
The lower-tail argument is identical, so a union bound yields
$\Pp(|\widehat\mu_0-\theta|>t)\leq\beta_H(t)$.

On the complementary event, the triangle inequality combines the population-dependent threshold with the observable correction magnitude
\[
|\widehat\theta_{\mathrm{RoVR}}-\theta|
\leq t+
|\widehat\theta_{\mathrm{RoVR}}-\widehat\mu_0|
=t+\widehat C_B.
\]
For each block,
$0\leq\sum_kJ_0(\cdot)\leq K$ and
$\sum_k\tau_k=K/2$, so its centered composite score has magnitude at most $K/2$. Together with $\widehat\nu\leq\nu_{\max}$, this gives
\[
\widehat C_B
\leq\frac{\nu_{\max}}{D_KW_B}\frac{BK}{2}
=C_B.
\]
The two probability statements in \eqref{eq:finite_rovr_bound} follow.
\end{proof}

\section{Proofs for Target, Credit, and Clean Efficiency}
\label{app:target_scale}

This appendix proves the target, bounded-credit, unequal-block oracle variance, scale-orthogonality, and clean transfer results. The arguments keep functional bias, contamination bias, finite-sample error, and credit shaping separate.

\subsection{Asymmetry and population contamination}

\begin{proof}[Proof of \Cref{prop:asymmetric_bias}]
Since $g(\theta_\rho)=0$, the mean-value theorem and the curvature lower bound give $|g(\mu)|\geq a_0|\mu-\theta_\rho|$, proving \eqref{eq:asymmetric_functional_bias}. For the third-moment statement, let $U=X-\mu$. The inequality $1-(1+x)^{-1/2}\leq x/2$ gives
\[
\left|\frac{U}{\sqrt{1+U^2/c^2}}-U\right|
\leq\frac{|U|^3}{2c^2}.
\]
Since $\E U=0$, $|\E\psi_c(U)|\leq\E|U|^3/(2c^4)$. The dimensionless curvature condition corresponds to $a_0=\bar a_0/c^2$, which proves \eqref{eq:third_moment_bias}.
\end{proof}

\begin{proof}[Proof of \Cref{prop:population_contamination}]
At $\theta_\epsilon=\theta_\rho(P_\epsilon)$,
\[
(1-\epsilon)\E_{P_0}\psi_c(X-\theta_\epsilon)
+\epsilon\E_Q\psi_c(X-\theta_\epsilon)=0.
\]
The contamination term has magnitude at most $M_c$, while the clean term has magnitude at least $a_0|\theta_\epsilon-\mu_0|$. Rearranging proves \eqref{eq:population_contamination_bias}.
\end{proof}

\subsection{Bounded-advantage stability and retained contrast}

\begin{proof}[Proof of \Cref{prop:credit_stability}]
The map $\chi_\kappa$ is 1-Lipschitz and lies in $[-\kappa,\kappa]$. Its square takes values in $[0,\kappa^2]$ and is $2\kappa$-Lipschitz. Each of the $r$ replaced rewards can change a squared-credit summand by at most $\kappa^2$; each unchanged reward changes by at most $2\kappa e_\theta$ through the center. The two quantities inside the scale square roots therefore differ by at most $r\kappa^2/G+2\kappa e_\theta$. Since the square root is $1/(2s_{\min})$-Lipschitz on $[s_{\min}^2,\infty)$, \eqref{eq:credit_scale_stability} follows.

For response $i$,
\[
|\chi_\kappa(R_i-\widehat\theta_R)
-\chi_\kappa(R_i^\star-\theta_R^\star)|
\leq\min\{e_{R,i}+e_\theta,2\kappa\}=d_{\chi,i}.
\]
Adding and subtracting the first numerator divided by the oracle scale gives
\[
|\widehat A_i-A_i^\star|
\leq\frac{d_{\chi,i}}{s_{\min}}
+\kappa\frac{|\widehat s_\chi-s_\chi^\star|}{s_{\min}^2},
\]
which is \eqref{eq:credit_stability}. The global magnitude bound follows from $|\chi_\kappa|\leq\kappa$. Sign preservation uses the sign-preserving map and the fact that an unchanged residual cannot cross zero when its margin exceeds $e_\theta$.

For \eqref{eq:credit_contrast_retention}, apply the mean-value theorem to the two residuals. On $[-M,M]$, the derivative is at least $m_{\kappa,M}$. The shared center cancels from their raw difference, and $\widehat s_\chi\leq\sqrt{s_{\min}^2+\kappa^2}$.
\end{proof}

\begin{proof}[Proof of \Cref{cor:finite_credit}]
On the event in \Cref{thm:finite_group}, the observed center is within $e_\theta(t)$ of the population reference $\theta$. Apply \Cref{prop:credit_stability} to the observed group and the clean precursor centered at $\theta$, then substitute $e_\theta(t)$ into \eqref{eq:credit_scale_stability} and \eqref{eq:credit_stability}. The theorem's tail probability gives the simultaneous statement.
\end{proof}

Equation \eqref{eq:credit_clean_fidelity} follows from
\[
|\chi_\kappa(u)-u|
=|u|\left\{1-(1+u^2/\kappa^2)^{-1/2}\right\}
\leq\frac{|u|^3}{2\kappa^2}.
\]

\subsection{Unequal-block oracle variance and scale orthogonality}

\begin{proof}[Proof of \Cref{prop:oracle_unequal_efficiency}]
For block $b$, standardization makes $Z_b=\sqrt{n_b}(\widetilde\mu_b-\theta)/\nu$ standard normal. The block composite score
\[
\xi_b=\sum_{k=1}^{K}[J_0(Z_b-\Delta_k)-\tau_k]
\]
has mean zero and variance $A_K$; ties have probability zero. Independence across blocks gives
\[
\operatorname{Var}\!\left(
\frac{\nu}{D_KW_B}\sum_b\xi_b
\right)
=\frac{\nu^2BA_K}{D_K^2W_B^2}
=\frac{\nu^2V_K}{N_{\mathrm{eff}}}.
\]
\end{proof}

\begin{proof}[Proof of \Cref{prop:scale_orthogonality}]
Because $\Phi(\Delta_k)=\tau_k$, $M_\Phi(0,1)=0$. Differentiation gives
\[
\partial_tM_\Phi(0,1)=\sum_k\phi(\Delta_k)=D_K,
\qquad
\partial_sM_\Phi(0,1)=\sum_k\Delta_k\phi(\Delta_k).
\]
The quantile grid is symmetric, so terms in the last sum cancel pairwise and the central term, if present, is zero. A second-order Taylor expansion on a fixed neighborhood proves \eqref{eq:orthogonal_expansion}.
\end{proof}

\subsection{Clean efficiency transfer}

\begin{proof}[Proof of \Cref{thm:efficiency_transfer}]
Let
\[
R_{B,n}(t,s)=\frac1B\sum_{b\notin\mathcal H_B}\sum_{k=1}^{K}
\{J_0(Z_{b,n}-t-s\Delta_k)-\tau_k\}.
\]
Each arbitrary-block composite score has magnitude at most $K/2$, so $|R_{B,n}(t,s)|\leq q_BK/(2B)=o(B^{-1/2})$. The exact decomposition is
\[
\widehat M_{B,n}(t,s)
=\frac{H_B}{B}M_n(t,s)+\mathbb U_{B,n}(t,s)+R_{B,n}(t,s).
\]
For honest blocks, (T3), \Cref{prop:scale_orthogonality}, and (T1)--(T2) give
\[
M_n(t_0,\widehat s)
=D_Kt_0+o_{\Pp}(B^{-1/2}),
\]
because $t_0^2=O_{\Pp}(B^{-1})$ and $(\widehat s-1)^2=o_{\Pp}(B^{-1/2})$. Moreover, $(H_B/B-1)M_n=-(q_B/B)O_{\Pp}(B^{-1/2})=o_{\Pp}(B^{-1/2})$. Condition (T4) therefore yields
\[
\widehat M_{B,n}(t_0,\widehat s)
=D_Kt_0+\mathbb U_{B,n}(0,1)+o_{\Pp}(B^{-1/2}).
\]
Substitution into \eqref{eq:rovr} gives
\[
\begin{split}
\widehat\theta_{\mathrm{RoVR}}-\theta_\rho
= {}&-\frac{\nu}{\sqrt nD_K}\mathbb U_{B,n}(0,1)
+\frac{\nu(1-\widehat s)t_0}{\sqrt n}\\
&-\frac{\nu(\widehat s-1)}{\sqrt nD_K}\mathbb U_{B,n}(0,1)
+o_{\Pp}((Bn)^{-1/2}).
\end{split}
\]
The second residual is $o_{\Pp}(B^{-1/2}/\sqrt n)$ because $(\widehat s-1)t_0=o_{\Pp}(B^{-3/4})$, and the third has the same order because $\mathbb U_{B,n}(0,1)=O_{\Pp}(B^{-1/2})$. This proves \eqref{eq:efficiency_linearization}.

The honest block score is bounded by $K$. Under (T5), its covariance converges to the Brownian-bridge covariance and its variance to $A_K$; also $H_B/B\to1$ because $q_B=o(\sqrt B)$. The triangular-array Lindeberg condition is automatic, so the central limit theorem gives variance $\nu^2A_K/D_K^2$. Finally,
\[
D_K/K\to\int_{\R}\phi^2(z)\,dz=\frac{1}{2\sqrt\pi},
\qquad
A_K/K^2\to\int_0^1\!\int_0^1[\min(u,v)-uv]\,du\,dv=\frac1{12},
\]
and their ratio tends to $\pi/3$.
\end{proof}

\section{Proofs for Smoothing and Policy Scope}
\label{app:policy}

This appendix proves the tie-aware smoothing, translation calibration, branch-count, policy-stability, leave-one-out, and credit-distortion results. It does not turn the robust token surrogate into an exact likelihood ratio or the bounded advantage into an unbiased raw-reward gradient.

\subsection{Smoothing and translation calibration}

\begin{proof}[Proof of \Cref{prop:smoothing}]
For $u\neq0$, $|H_\gamma(u)-J_0(u)|\leq e^{-|u|/\gamma}$; at $u=0$, both equal $1/2$. Summing at most $BK$ nonzero errors and multiplying by the prefactor $\widehat\nu/(D_KW_B)$ proves \eqref{eq:smoothing_margin_bound}. Integrating the logistic-step discrepancy over a continuous margin gives $2\gamma\log2$, hence the stated $O(\gamma)$ expected error under a bounded continuous density. Atomic mass at zero has zero discrepancy.
\end{proof}

To verify \eqref{eq:rovr_lipschitz}, write $M(T)=\median_bT_b$ and $u_{bk}(T,\lambda)=T_b-M(T)-\lambda\Delta_k/\sqrt n$. The sample median is 1-Lipschitz in input sup norm, so the raw residuals satisfy
\[
\max_{b,k}|u_{bk}(T,\lambda)-u_{bk}(T',\lambda')|
\leq2d_T+\frac{\Delta_{\max}d_\lambda}{\sqrt n}.
\]
For $F_\gamma(T,\lambda)=B^{-1}\sum_{b,k}[H_\gamma(u_{bk}(T,\lambda))-\tau_k]$, symmetry of the quantile grid gives $\sum_k\tau_k=K/2$ and hence $|F_\gamma|\leq K/2$. Also, $|H_\gamma'|\leq1/(4\gamma)$ implies
\[
|F_\gamma(T,\lambda)-F_\gamma(T',\lambda')|
\leq\frac{K}{4\gamma}
\left(2d_T+\frac{\Delta_{\max}d_\lambda}{\sqrt n}\right).
\]
Finally, $\mathcal S_\gamma(T,\lambda)=M(T)-\lambda F_\gamma(T,\lambda)/(\sqrt nD_K)$. Changing the initializer, scale prefactor, and logistic arguments in turn, and using $\lambda,\lambda'\leq\lambda_{\max}$, gives \eqref{eq:rovr_lipschitz} for the raw-residual parameterization used by the algorithm.

For \eqref{eq:rovr_translation_calibration}, translating all token values translates each local convex root and every equivariantly initialized smooth median by the same amount. Residuals, scales, and correction scores are unchanged. Differentiating the translation identity with respect to the common shift gives the gradient-sum identity.

\subsection{Branch counts and local policy stability}

\begin{proof}[Proof of \Cref{prop:branch_flip_count}]
If $|\widehat m_i-m_i^\star|\leq e_{m,i}$ and the oracle log ratio is farther than $e_{m,i}$ from both log boundaries, the observed and oracle log ratios lie in the same interval cut out by those boundaries. A union count gives the first inequality. The center sign statement follows because an unchanged residual can cross zero only if its oracle magnitude is at most the center displacement. Summing gives \eqref{eq:flip_count_bounds}.
\end{proof}

\begin{proof}[Proof of \Cref{prop:policy_stability}]
For nonnegative ratios, $g(q,A)$ is $\overline Q$-Lipschitz in $A$ over the stated domain and $|A|$-Lipschitz in $q$. The exponential mean-value bound gives $|\widehat q-q^\star|\leq e^Le_m$, proving \eqref{eq:policy_value_stability}.

Away from clipping kinks and under the same active branch, the gradient is either zero or $Aq\nabla m$. The log-aggregate conditions imply
\[
\|\widehat q\nabla\widehat m-q^\star\nabla m^\star\|
\leq e^L(e_J+J_{\max}e_m),
\qquad
\|\widehat q\nabla\widehat m\|
\leq e^L(J_{\max}+e_J).
\]
Adding and subtracting $A^\star\widehat q\nabla\widehat m$ proves \eqref{eq:policy_gradient_stability}.
\end{proof}

\subsection{Linear leave-one-out direction and bounded-advantage distortion}

\begin{proof}[Proof of \Cref{thm:crossfit_baseline}]
Condition on $(x,Y_{-i},U)$. Then $B_i,S_i$ are constant with respect to $Y_i$ and
\[
\E[\nabla_\theta\log\pi_\theta(Y_i\mid x)B_i/S_i\mid x,Y_{-i},U]=0.
\]
The reward term equals $S_i^{-1}\nabla_\theta\E[R(x,Y)\mid x]$. Averaging over $Y_{-i},U$ and then over $i$ gives \eqref{eq:loo_scaled_direction}. Exact likelihood weighting changes old-policy sampling to the same on-policy expectation.
\end{proof}

\begin{proof}[Proof of \Cref{prop:credit_gradient_distortion}]
Equation \eqref{eq:credit_clean_fidelity} gives $|\chi_\kappa(U_i)-U_i|\leq|U_i|^3/(2\kappa^2)$. Multiply by the score norm, use $S_i^{-1}\leq s_{\min}^{-1}$, and apply the triangle inequality inside expectation to prove \eqref{eq:credit_gradient_distortion}.
\end{proof}

}

\section{Experimental Protocol and Hyperparameters}
\label{app:exp_protocol}

This appendix specifies the training configurations, RoVR components, and statistical reporting conventions for \Cref{sec:experiments}. The channel definitions follow \Cref{sec:robust_scale_exp}; the offline studies evaluate the corresponding update statistics on saved responses.

\subsection{Training configuration}

The mathematical-reasoning experiments use Qwen3-4B-Base and OLMo-3-7B-Instruct on a 7,500-example MATH split, with group size $G=16$, actor learning rate $10^{-6}$, prompt/response limits 2,048/8,192, global and mini-batches of 192/24 prompts, asynchronous vLLM rollout with tensor parallelism 2, GPU memory utilization 0.8, and at most 500 steps over 10 epochs. The GovReport runs use the same backbones, sample eight summaries per input with temperature and top-$p$ equal to 1.0, and train with LongDocFACTScore.

The tool-call comparison uses the identical 3,920/80 train/test parquet files and rule scorer in both runs, with Qwen3-4B-Base, $G=16$, actor learning rate $10^{-6}$, prompt/response limits 3,072/1,024, global and mini-batches of 64/16 prompts, and asynchronous vLLM rollout with tensor parallelism 2. Validation uses one greedy response per prompt every 20 steps. The only recorded configuration differences are experiment names and output directories; the RoVR-GSPO launcher selects block-credit advantages and SoftRoVR sequence weights, while the GSPO launcher selects ordinary group mean/standard-deviation advantages and arithmetic sequence weights. Both configurations request 500 steps, but the supplied logs contain training only through the step-100 checkpoint and validation through step 80. The run identifiers are \texttt{sqy23z50} (RoVR-GSPO) and \texttt{y1c3z2ii} (GSPO).

The dual-channel Qwen3-4B-Base run used in \Cref{sec:softrovr_token_stress} records reward parameters $\kappa=1$ and $s_{\min}=10^{-3}$. Its ratio channel uses $B=8$ balanced contiguous blocks, a minimum block size of $4$, $K=9$ quantiles, $c=1$, $\gamma=\eta=0.01$, $S=32$ safeguarded solver iterations, $a_{\min}=\nu_{\min}=10^{-6}$, $\nu_{\max}=10$, and denominator regularization $10^{-8}$. The recorded clipping interval is $[0.9997,1.0004]$. These settings describe the trajectory used by the paired offline comparison.

The reward-contamination analysis uses groups of $G=16$, the full RoVR reference with local scale $c=1$, credit parameters $\kappa=1$ and $s_{\min}=10^{-3}$, bootstrap seed $20250904$, and $500$ bootstrap replicates. The global reward standard deviation is $1.165$, and the median within-group population standard deviation is approximately $0.165$. A $4\sigma$ global shift therefore corresponds to about $28.28$ times this typical within-group scale.

\subsection{Resolution and reporting conventions}

The RoVR configurations in the experiments use blockwise M-centers and the composite-quantile outer correction. The reward-reference ablation retains this construction and changes only the residual normalization; the channel-placement ablations enable the reward channel, the ratio channel, or both. The theoretical conditions $B\geq2q+1$ and $n_{\min}\geq2s+1$ specify the contamination budgets covered by the two-level guarantee. \Cref{app:simulation} studies six scalar reference estimators under matched partitions.

Downstream comparisons match initialization, sampled-response budget, decoding, optimizer budget, and evaluation code; prompt order is not verified across the two tool-call runs. \Cref{tab:math_main,tab:placement_ablation} report means and cross-run standard deviations over three independent runs. GovReport reports single-run point estimates; the tool-call comparison reports one run per method and selects the highest Total reward from each run's four logged validations. The token-ratio analysis uses $5{,}000$ prompt-cluster bootstrap replicates from a fixed training trajectory, the reward analysis uses $500$ group-bootstrap replicates, and the scalar simulations quantify Monte Carlo uncertainty. The GovReport figure displays reward trajectories and their visual separation.

\subsection{Additional token-ratio stress diagnostics}
\label{app:softrovr_stress}

The primary token-ratio stress test in \Cref{sec:softrovr_token_stress} uses positive perturbations in the main figure; here we apply the same protocol to negative perturbations. At step $18$, the negative $16\sigma_\ell$ spike reduces $D_q$ by $26.09\%$, while the flip rate decreases from $21.061\%$ to $17.370\%$ with difference $3.691\%$ and interval $[2.949\%,4.453\%]$. For a negative $20\%$ burst, $D_q$ decreases by $87.39\%$ and the flip-rate difference is $0.807\%$ with interval $[0.202\%,1.413\%]$. The corresponding positive-burst interval includes zero.

For a stage-sensitivity check, every complete recorded step is evaluated at spike magnitude $16\sigma_\ell$ and burst fraction $20\%$ using a fixed seeded sample of $64$ complete prompt groups (1,024 responses) and five paired positions per response. Across all $35$ complete steps and both directions, the paired differences for $D_q$ and strong-spike flips are positive, with pointwise intervals above zero. For $20\%$ bursts, the reduction in mean $D_q$ ranges from $85.29\%$ to $90.12\%$. The burst flip difference changes sign across steps: its point estimate favors SoftRoVR in $28/35$ positive-direction checks and $20/35$ negative-direction checks. These results support lower aggregation sensitivity across recorded stages, while clipping effects remain direction- and stage-dependent. These steps are observations from one training run, not independent seeds, and the pointwise intervals are not adjusted for multiple comparisons.
\begin{figure}[!htbp]
	\centering
	\includegraphics[width=0.98\linewidth]{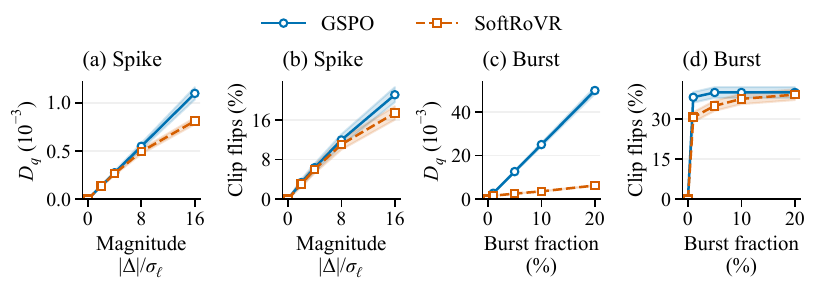}
	\caption{Negative-direction sensitivity analysis at step $18$ ($192$ prompt groups, $3{,}072$ responses). Bands are pointwise 95\% prompt-cluster bootstrap intervals. The perturbations are $-c\sigma_\ell$ for spikes and $-8\sigma_\ell$ for bursts. SoftRoVR reduces $D_q$ throughout the tested nonzero settings. Its clipping benefit is larger for strong spikes and short bursts than for the longest burst.}
	\label{fig:softrovr_negative_stress}
\end{figure}

\begin{figure}[!htbp]
\centering
\includegraphics[width=0.98\linewidth]{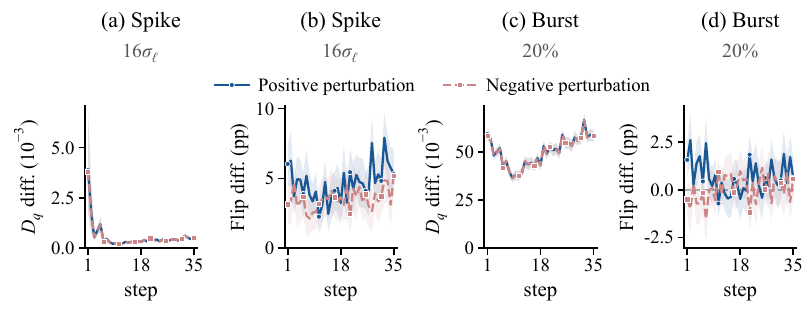}
\caption{Paired differences across $35$ complete steps for spike $16\sigma_\ell$ and burst $20\%$ stress, using $64$ prompt groups per step. Lines connect the recorded estimates without smoothing; shaded bands are pointwise 95\% prompt-cluster bootstrap intervals. Positive values favor SoftRoVR. Effects on log-weight displacement are consistent across the recorded stages; long-burst clipping effects are stage-dependent.}
\label{fig:softrovr_cross_step}
\end{figure}

\rovrReferenceEstimatorTable

\section{Implementation Details}
\label{app:implementation}

This appendix specifies the numerical object corresponding to the definitions in the main text. It records the full Barron family for completeness, then gives the convex solver, tie convention, balanced-block handling, SoftRoVR path, and complexity. Main results use only $\alpha=1$.

\rovrAdvantageVariantTable

\subsection{Estimator and update summaries}

\rovrUpdateAlgorithm

\subsection{Barron family and convex root solver}

For residual $u$, shape $\alpha$, and scale $c>0$, Barron's loss is~\citep{BarronCVPR2019}
\[
\rho(u;\alpha,c)=
\begin{cases}
\frac12(u/c)^2, & \alpha=2,\\
\log(1+\frac12(u/c)^2), & \alpha=0,\\
1-\exp[-\frac12(u/c)^2], & \alpha=-\infty,\\
\frac{|\alpha-2|}{\alpha}
\left[\left(1+\frac{(u/c)^2}{|\alpha-2|}\right)^{\alpha/2}-1\right], & \text{otherwise.}
\end{cases}
\]
For $\alpha=1$, safeguarded Newton or bisection solves the strictly decreasing empirical score. A bracket is given by the sample minimum and maximum. Newton steps leaving the bracket are replaced by bisection, and termination uses both score and interval tolerances.

Shapes $\alpha<1$ are redescending and may have multiple stationary points. If reported, they use a fixed median initialization, a declared graduated schedule, and a separate label; no convex-root or finite-group theorem is transferred to them.

\subsection{Blocks, ties, scales, and differentiation}

The entire balanced random assignment or permutation, not merely the block sizes, is sampled independently of corruption labels whenever \Cref{prop:random_partition} is invoked. The exact $n_b$, partition seed, and whether corruption could adapt to the partition are logged. For $B=1$, the composite correction is identically zero by \Cref{prop:calibration}. For even $B$, the midpoint median is used consistently in hard reward code; the token map uses the unique finite-temperature initializer in \eqref{eq:smooth_median}.

For exact discrete rewards, $J_0$ assigns half weight to equality. A numerical implementation may use a fixed tolerance $\epsilon_{\rm tie}$ to assign $1/2$ when $|u|\leq\epsilon_{\rm tie}$; this convention is part of the estimator specification. The logistic map returns $1/2$ at zero.

For small local blocks, an independent pilot or a lagged running statistic can supply the scale. Gradients are stopped through reward-side centering, scale estimation, and partitioning. On the ratio side, gradients pass through the unrolled safeguarded solvers and the finite-temperature correction. The pooled smooth scale uses capped local sandwich scales; derivative checks should evaluate the same finite-temperature map.

With $S$ root iterations, local centers cost $O(SN)$ scalar operations. Hard median selection costs $O(B)$, while the general outer correction costs $O(BK)$. SoftRoVR's two smooth median solves add $O(SB)$ operations; bounded reward-credit construction costs $O(G)$. Forward aggregation stores local statistics and correction terms, while differentiating through the unrolled solves additionally retains token-level intermediate states. Actual memory usage depends on solver unrolling, token packing, and the autograd implementation.

\FloatBarrier
\section{Supplementary Experiment: Reference-Estimator Variance and Robustness}
\label{app:simulation}

This appendix studies the scalar reference estimators under controlled sampling and contamination. Each method receives the same observations and balanced partition, with uncontaminated center $0$ as the target. The design isolates reference-estimation variance and error, complementing the downstream policy comparisons.

\subsection{Design and estimands}

We use $N=128$ observations partitioned into $B=8$ balanced blocks of size $16$, with the RoVR parameters $K=9$ and $c=1$. The clean distributions are standard Gaussian and $t_3/\sqrt{3}$, so both have center and variance equal to zero and one, respectively. For the contamination experiments, $16$ observations (12.5\%) are shifted by $+8$. In the dispersed condition these observations are allocated as evenly as possible across the eight blocks; in the coherent condition one entire block is shifted. Thus the two geometries have the same contamination budget, while only their placement relative to the block interface differs. The estimator comparison uses 3,000 independent trials, the geometry sweep uses 2,000 trials per shift, and the oracle outer-factor experiment uses 20,000 trials. All random streams are seeded with seed $0$.

We compare the arithmetic mean, the global convex M-center, MOM, a matched \emph{VRMOM-style} estimator, Robust MOM, and RoVR. The VRMOM-style comparator uses arithmetic local means, composite-quantile correction, and the median block standard-deviation scale. We report sampling variance and RMSE to the uncontaminated center. Error bars are 95\% bootstrap intervals over Monte Carlo trials.

\subsection{Oracle outer-factor validation}

Table~\ref{tab:sim_outer_factor} and Figure~\ref{fig:sim_outer_factor} validate the outer aggregation in isolation. The theoretical factor $V_K$ decreases from $\pi/2$ at $K=1$ toward $\pi/3$ as the quantile grid is refined. At the paper's setting $K=9$, $V_K=1.0691$ and the empirical oracle factor is $1.0747$; at $K=31$, they are $1.0498$ and $1.0540$. The empirical median factor is $1.5649$, close to $\pi/2=1.5708$. The finite-sample agreement supports the implementation of the composite-quantile correction and its stated outer-factor interpretation. The complete estimator additionally incorporates the local sandwich scale, finite block sizes, and estimated initialization.

\begin{table}[!htbp]
\centering
\caption{Oracle outer-factor validation. The empirical factors use 20,000 trials with $B=101$ and block size $n=32$. The median factor is shown for reference; $\pi/2$ and $\pi/3$ are the clean asymptotic reference values.}
\label{tab:sim_outer_factor}
\setlength{\tabcolsep}{7pt}
\begin{tabular}{rrrr}
\toprule
$K$ & $V_K$ (theory) & Oracle correction (emp.) & Outer median (emp.) \\
\midrule
1  & 1.5708 & 1.5812 & 1.5649 \\
3  & 1.1680 & 1.1802 & 1.5649 \\
5  & 1.1034 & 1.1102 & 1.5649 \\
9  & 1.0691 & 1.0747 & 1.5649 \\
15 & 1.0564 & 1.0614 & 1.5649 \\
31 & 1.0498 & 1.0540 & 1.5649 \\
\midrule
\multicolumn{2}{r}{$\pi/2$} & \multicolumn{2}{r}{1.5708} \\
\multicolumn{2}{r}{$\pi/3$} & \multicolumn{2}{r}{1.0472} \\
\bottomrule
\end{tabular}
\end{table}

\begin{figure}[!htbp]
\centering
\includegraphics[width=0.8\linewidth]{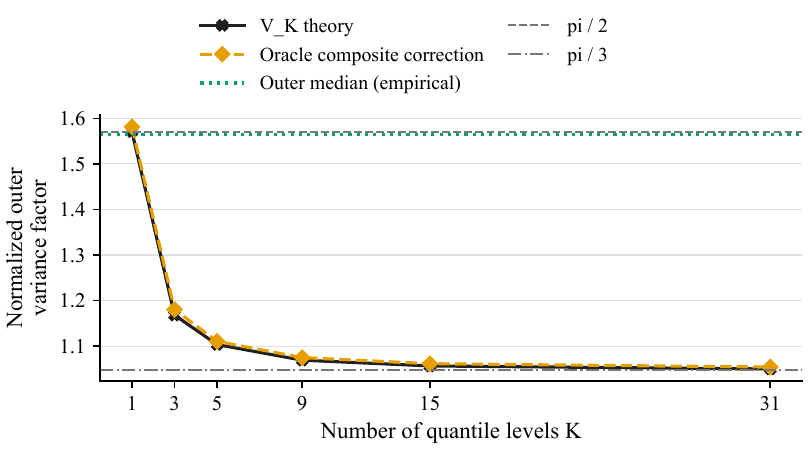}
\caption{Oracle outer-factor validation. The oracle composite-quantile estimator follows $V_K$ closely, while the outer median remains near $\pi/2$.}
\label{fig:sim_outer_factor}
\end{figure}

\begin{figure}[!htbp]
	\centering
	\includegraphics[width=0.98\linewidth]{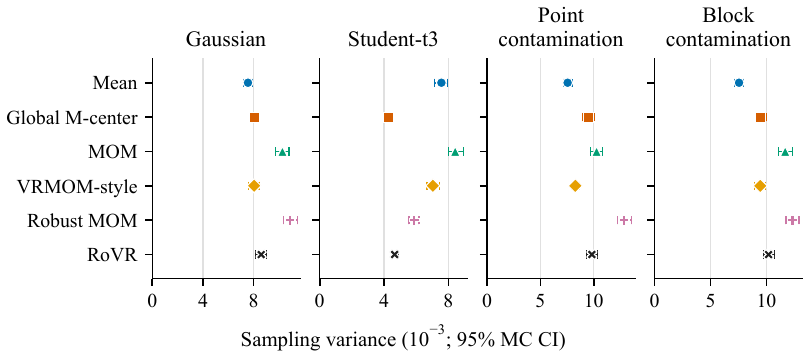}
	\caption{Sampling variance across clean and contaminated scenarios. Error bars are 95\% bootstrap Monte Carlo intervals over 3,000 estimator trials. Under contamination, variance alone is not a sufficient robustness metric because it does not include systematic displacement.}
	\label{fig:sim_variance}
\end{figure}
\subsection{Finite-sample comparison}

The sampling variances are reported in Table~\ref{tab:sim_variance} and Figure~\ref{fig:sim_variance}. Under clean Gaussian data, the arithmetic mean is the most efficient estimator in this comparison ($0.007552$), while RoVR has variance $0.008596$. This is expected: bounded local scores and block aggregation trade some Gaussian efficiency for robustness. RoVR nevertheless has lower variance than MOM ($0.010271$) and Robust MOM ($0.010865$) in this finite configuration. Under the unit-variance $t_3$ distribution, robust score estimators benefit from suppressing tail observations; RoVR has variance $0.004660$, close to the global M-center ($0.004274$) and below the arithmetic mean ($0.007567$) and MOM ($0.008424$).

\begin{table}[!htbp]
\centering
\caption{Sampling variance of the six reference estimators. Each entry is computed from 3,000 trials; the full CSV also reports bootstrap Monte Carlo intervals.}
\label{tab:sim_variance}
\setlength{\tabcolsep}{3.5pt}
\renewcommand{\arraystretch}{1.08}
\resizebox{\linewidth}{!}{%
\begin{tabular}{lrrrrrr}
\toprule
Scenario & Mean & Global M & MOM & VRMOM-style & Robust MOM & RoVR \\
\midrule
Gaussian & 0.007552 & 0.008088 & 0.010271 & 0.008033 & 0.010865 & 0.008596 \\
Student-$t_3$ & 0.007567 & 0.004274 & 0.008424 & 0.007038 & 0.005871 & 0.004660 \\
Point contamination & 0.007552 & 0.009483 & 0.010271 & 0.008273 & 0.012848 & 0.009828 \\
Block contamination & 0.007552 & 0.009486 & 0.011677 & 0.009443 & 0.012323 & 0.010190 \\
\bottomrule
\end{tabular}}
\end{table}

\subsection{Contamination geometry and interpretation}

\begin{table}[H]
\centering
\caption{RMSE to the uncontaminated center $0$ under the fixed $+8$ contamination. Both contamination geometries modify 16 of 128 observations.}
\label{tab:sim_rmse}
\setlength{\tabcolsep}{3.5pt}
\renewcommand{\arraystretch}{1.08}
\resizebox{\linewidth}{!}{%
\begin{tabular}{lrrrrrr}
\toprule
Scenario & Mean & Global M & MOM & VRMOM-style & Robust MOM & RoVR \\
\midrule
Point contamination & 1.0059 & 0.2719 & 1.0067 & 1.0064 & 0.2746 & 0.2695 \\
Block contamination & 1.0059 & 0.2718 & 0.1171 & 0.1129 & 0.1206 & 0.1170 \\
\bottomrule
\end{tabular}}
\end{table}

Table~\ref{tab:sim_rmse} and Figure~\ref{fig:sim_contamination_rmse} show why the failure geometry must be declared. With dispersed point contamination, the mean, MOM, and VRMOM-style estimator have RMSE approximately $1.006$, whereas RoVR has RMSE $0.270$, close to the global M-center ($0.272$) and Robust MOM ($0.275$). The bounded local score prevents the shifted observations from directly controlling every block mean. With one fully corrupted block, block-based estimators are substantially more effective: MOM, VRMOM-style, and RoVR have RMSE $0.117$, $0.113$, and $0.117$, respectively, while the global M-center has RMSE $0.272$. Both RoVR and the VRMOM-style comparator benefit from the block interface in this setting.

The shift sweep in Figures~\ref{fig:sim_dispersed_sweep} and~\ref{fig:sim_block_sweep} gives the same qualitative picture over contamination magnitudes from $0$ to $16$. Under dispersed corruption, the mean-like block estimators grow approximately linearly with the shift, while the global and block-robust estimators saturate. Under coherent block corruption, the explicit block interface allows the median and composite-quantile outer aggregators to reject the bad block for the declared one-block failure geometry. RoVR inherits this behavior through its robust local centers and composite-quantile correction. The sweep evaluates the declared, fixed block-failure geometry across perturbation magnitudes.

The simulations reproduce the predicted outer-efficiency factor and show how local and block robustness contribute under different contamination geometries. RoVR improves over median-based aggregation in the clean settings and achieves low error under both dispersed and coherent contamination. The downstream effects of the complete optimizer are evaluated in \Cref{tab:math_main,tab:govreport_rouge,tab:placement_ablation}; the exploratory tool-call proxy is reported in \Cref{sec:tool_call_results}.

\begin{figure}[!htbp]
\centering
\includegraphics[width=0.8\linewidth]{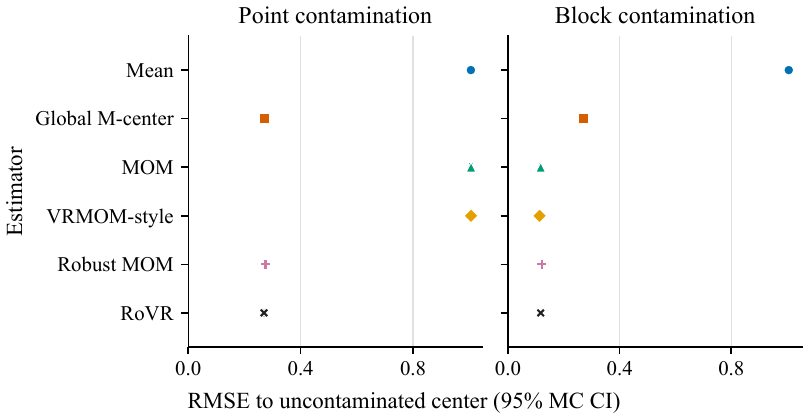}
\caption{RMSE to the uncontaminated center under the two matched contamination geometries. Error bars are 95\% bootstrap Monte Carlo intervals over 3,000 trials.}
\label{fig:sim_contamination_rmse}
\end{figure}

\begin{figure}[!htbp]
\centering
\includegraphics[width=0.8\linewidth]{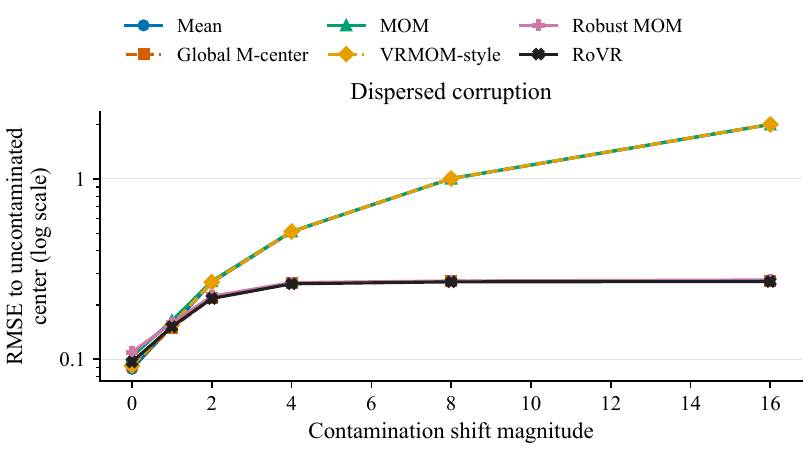}
\caption{RMSE sweep for dispersed point corruption. The y-axis is logarithmic to show both mean-like and robust estimators over the same range.}
\label{fig:sim_dispersed_sweep}
\end{figure}

\begin{figure}[!htbp]
\centering
\includegraphics[width=0.8\linewidth]{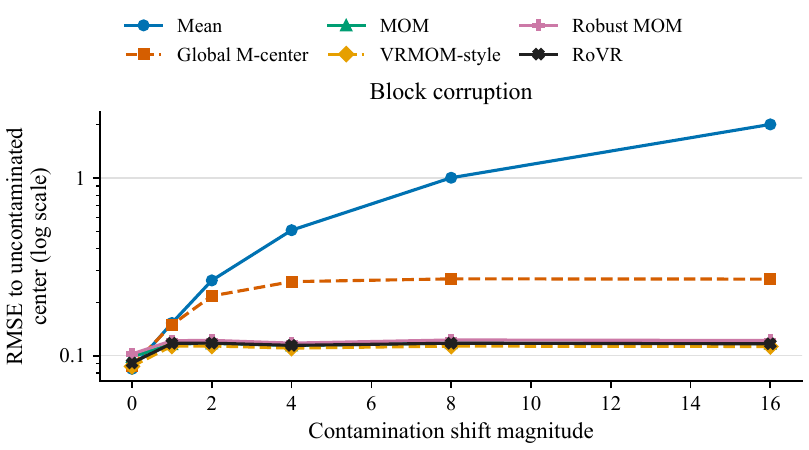}
\caption{RMSE sweep for one-block coherent corruption. The favorable behavior of block estimators is conditional on the declared block failure geometry.}
\label{fig:sim_block_sweep}
\end{figure}


\rovrTheoryAppendix

\end{document}